\documentclass{article}
\usepackage[utf8]{inputenc}
\usepackage[T1]{fontenc}
\usepackage{times}
\usepackage{helvet}
\usepackage{courier}

\usepackage{amsmath}
\usepackage{amssymb}
\usepackage{amsfonts}

\usepackage{booktabs}
\usepackage{multirow}
\usepackage{makecell}
\usepackage{array}
\usepackage{tabularx}

\usepackage{graphicx}
\usepackage{xcolor}
\usepackage{caption}

\usepackage[linesnumbered,ruled,vlined]{algorithm2e}

\usepackage{enumitem}
\usepackage{tcolorbox}
\usepackage{float}
\usepackage{nicefrac}
\usepackage{microtype}

\usepackage{url}
\usepackage[colorlinks=true, linkcolor=blue, citecolor=blue, urlcolor=blue]{hyperref}

\usepackage{tikz}

\usepackage[numbers,square]{natbib}

\usepackage{arxiv}

\usepackage{amsthm}

\newtheorem{theorem}{Theorem}[section]
\newtheorem{lemma}[theorem]{Lemma}

\title{Missing-by-Design: Certifiable Modality Deletion for Revocable Multimodal Sentiment Analysis}

\author{
    Rong Fu \\
    Independent Researcher \\
    Corresponding author \and
    Ziming Wang \\
    Independent Researcher \and
    Chunlei Meng \\
    Independent Researcher \and
    Jiekai Wu \\
    Independent Researcher \and
    Kangan Qian \\
    Independent Researcher \and
    Hao Zhang \\
    Independent Researcher \and
    Simon Fong \\
    Independent Researcher
}

\renewcommand{\headeright}{}
\renewcommand{\undertitle}{}
\renewcommand{\shorttitle}{Missing-by-Design}

\hypersetup{
    pdftitle={Missing-by-Design: Certifiable Modality Deletion for Revocable Multimodal Sentiment Analysis},
    pdfsubject={cs.CL, cs.LG},
    pdfauthor={Rong Fu, Ziming Wang, Chunlei Meng, Jiekai Wu, Kangan Qian, Hao Zhang, Simon Fong},
    pdfkeywords={Multimodal sentiment analysis, missing modality, certifiable deletion, privacy preserving, property embedding, modality reconstruction},
    pdfstartview={FitH},
    colorlinks=true,
    linkcolor=red,
    citecolor=green,
    filecolor=magenta,
    urlcolor=cyan
}

\begin{document}
\maketitle

\begin{abstract}
As multimodal systems increasingly process sensitive personal data, the ability to selectively revoke specific data modalities has become a critical requirement for privacy compliance and user autonomy. We present Missing-by-Design (MBD), a unified framework for revocable multimodal sentiment analysis that combines structured representation learning with a certifiable parameter-modification pipeline. Revocability is critical in privacy-sensitive applications where users or regulators may request removal of modality-specific information. MBD learns property-aware embeddings and employs generator-based reconstruction to recover missing channels while preserving task-relevant signals. For deletion requests, the framework applies saliency-driven candidate selection and a calibrated Gaussian update to produce a machine-verifiable Modality Deletion Certificate. Experiments on benchmark datasets show that MBD achieves strong predictive performance under incomplete inputs and delivers a practical privacy–utility trade-off, positioning surgical unlearning as an efficient alternative to full retraining.
\end{abstract}


\section{Introduction}
Multimodal sentiment analysis aims to integrate complementary cues from text, audio and visual streams to infer human affect and sentiment in real world scenarios. Prior studies demonstrate that multimodal representations improve predictive accuracy and robustness when modalities are jointly available. However, practical systems often face partial observability: modalities can be missing or corrupted due to privacy choices, sensor faults, automatic speech recognition errors, or collection constraints. Models trained on fully observed multimodal data frequently lose performance and reliability when confronted with such incompleteness. Empirical and survey works highlight both the prevalence of missing modalities and the limitations of existing approaches for handling them \cite{zhan2025systematic,pham2019found,aguilar2019multimodal}.

A range of strategies has been proposed to increase resilience under missing modalities. Some methods focus on strengthening the available modality embeddings via robust representation learning or self-distillation. Other approaches explicitly reconstruct absent channels using learned priors or generative models. Structural techniques such as graph based completion exploit turn and speaker dependencies in conversational data, and modality reweighting schemes attempt to rebalance contributions from underrepresented channels \cite{li2023enhancing,wang2023incomplete,lian2023gcnet,xu2023grmi,nguyen2024ada2i}. These lines of work show complementary strengths but also reveal two persistent gaps. First, many methods rely exclusively on the observed modalities at prediction time and thus fail to leverage modality-level distributional priors that capture sample-invariant characteristics. Second, while several generative or imputation schemes use a single learned distribution for reconstruction, they often overlook intrinsic differences between modalities that should inform modality-specific reconstructions.

In parallel, the need for operational privacy controls has driven research on machine unlearning and targeted deletion. Applications in sensitive domains, notably healthcare, motivate solutions that can remove modality-specific information from a trained model without full retraining \cite{liu2024patient,rahman2024survey,zhang2021privacy,fabiano2025affective}. Existing unlearning methods emphasize weight scrubbing, calibrated noise injection, or Newton-style updates, but extending these techniques to heterogeneous multimodal backbones and providing verifiable deletion guarantees remains an open challenge.

To address the dual objectives of robust fusion under missing inputs and verifiable modality-level revocation, we propose Missing-by-Design (MBD). MBD combines property-aware decomposition and modality-specific generators with a numerically stable candidate selection and calibrated surgery operator that produces a Modality Deletion Certificate. The property embeddings capture modality-level priors that guide reconstruction and stabilize fusion. The surgery operator uses saliency and a SwiftPrune inspired importance proxy to identify parameter subsets for safe modification, followed by Gaussian calibration that controls indistinguishability relative to a model never exposed to the target modality.

Our contributions are as follows. First, we introduce MBD, a practical pipeline that unifies property-aware representation decomposition, contrastive back-translation objectives, and a certifiable parameter surgery step for modality-level deletion. Second, we design a property embedding mechanism that separates sample-invariant modality characteristics from sample-specific signals, and we pair it with dedicated generator and back-translation networks to produce high-fidelity reconstructions for absent channels. Third, we propose a numerically stable importance proxy and a sensitivity-aware candidate selection strategy that, together with Gaussian calibration, yield a machine-verifiable Modality Deletion Certificate while preserving downstream utility. Finally, we present empirical evaluations on standard multimodal benchmarks that demonstrate improved robustness under diverse missing-modality regimes and show how MBD provides a tunable privacy-utility envelope for modality revocation.

\begin{figure*}[t]
  \centering
  \includegraphics[width=0.98\textwidth]{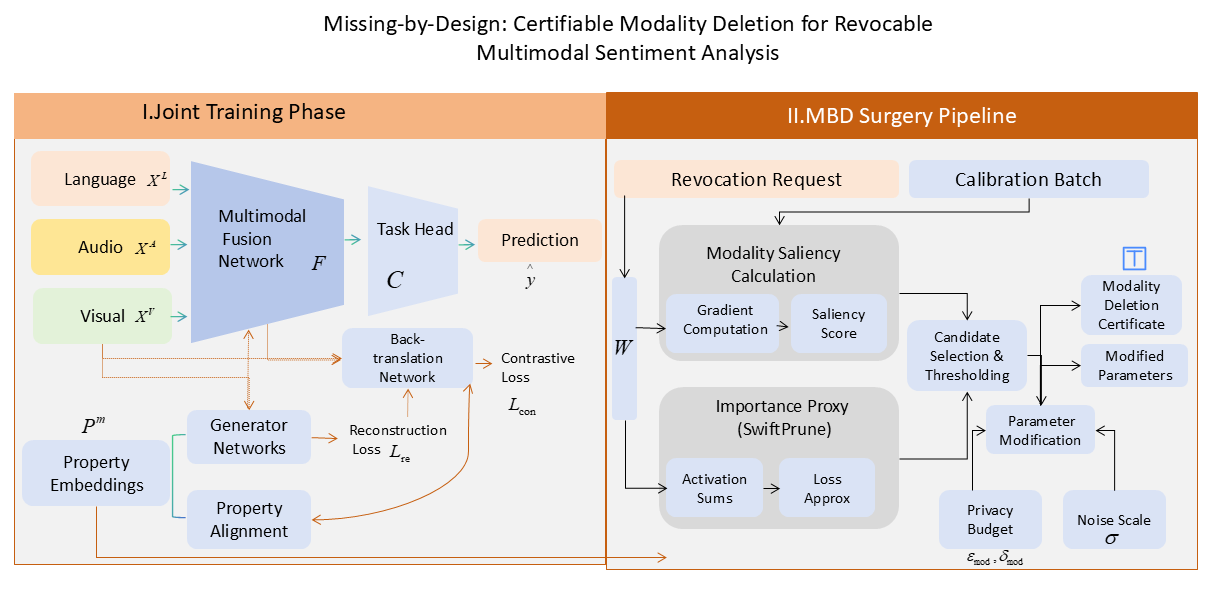}
  \caption{Overview of the Missing-by-Design (MBD) framework for certifiable modality deletion. The architecture is organized into two primary stages: a property-informed joint training phase and a weight surgery pipeline. In the training phase, multimodal inputs ($X^L, X^A, X^V$) are integrated via a fusion network $\mathcal{F}$ for sentiment prediction, while auxiliary generator networks $\mathcal{G}_m$ and property embeddings $P^m$ are optimized to enforce modality-specific reconstruction and property alignment. Upon a revocation request for modality $m^\star$, the surgery pipeline utilizes a calibration batch $\mathcal{B}$ to compute the modality saliency $s_q^{(m^\star)}$ and a SwiftPrune-inspired importance proxy $L_q$. These metrics guide the candidate selection and thresholding process ($\eta_s, \eta_L$), followed by a differential-privacy calibrated Gaussian mechanism ($\varepsilon_{\mathrm{mod}}, \delta_{\mathrm{mod}}$) for parameter modification. The pipeline ultimately outputs the modified model parameters $\mathcal{W}'$ alongside a machine-verifiable Modality Deletion Certificate (MDC).}
  \label{fig:mbd_architecture}
\end{figure*}

\section{Related Work}

\subsection{Multimodal sentiment and emotion modelling}
Multimodal sentiment analysis and emotion recognition leverage complementary cues from text, audio and visual streams to improve predictive performance. Unified frameworks that jointly model sentiment and emotion demonstrate benefits from shared latent spaces and task-aware supervision, as shown by recent unified formulations that align label and feature representations \cite{hu2022unimse}. Transformer-based fusion architectures and emotion-level representation schemes further refine cross-modal interactions, improving multi-label emotion recognition by integrating fine-grained token and frame alignments \cite{le2023multi,khan2025memocmt}. Cross-modal alignment and attention enhancements have been applied to temporal video tasks to better capture emotion dynamics \cite{xiao2023cross,fan2023pmr}. These contributions illustrate that careful fusion and representation design are central to robust multimodal affective modelling \cite{hu2022unimse,gao2024enhanced}. Empirical evidence on short e-commerce reviews suggests that transformer-based embeddings can be outperformed by classical methods such as Word2Vec in clustering tasks, particularly when insufficient text length limits effective contextual modeling \cite{lai2026transformers}.

\subsection{Handling missing and noisy modalities}
Practical systems must tolerate absent or corrupted modalities encountered in real-world data streams. Graph-based completion techniques reconstruct missing conversational modalities by exploiting structural dependencies among modalities and turns \cite{lian2023gcnet}. Two-stage schemes first denoise modality-specific signals and then complement missing channels using learned priors; such denoise-then-complement strategies have proven effective on noisy benchmarks \cite{zhuang2025tmdc,wei2025msaf}. Meta-learning approaches enable a single model to generalize across varying missing-rate regimes by quickly adapting to new incompleteness patterns \cite{tu2025meta}. Proxy-driven mechanisms and latent-Gaussian modelling capture uncertainty in absent channels and support robust downstream fusion \cite{zhu2025proxy,wang2023incomplete}. Collectively, these methods motivate architectures that either synthesize absent information or emphasize invariant representations that resist corruption \cite{xiang2024multimodal}.

\subsection{Representation learning and contrastive strategies}
Recent work emphasizes representation-level advances to make fusion resilient and discriminative. Contrastive objectives applied at global and local scales encourage modality-invariant factors while preserving salient, task-relevant signals \cite{mai2023learning,yang2023confede}. Decomposition-based pipelines separate modality-common and modality-specific components to tighten alignment and reduce redundancy during fusion \cite{zeng2024disentanglement,liu2024contrastive}. Prototypical rebalancing promotes class-centric clustering, which helps underrepresented modalities contribute more effectively to the joint embedding \cite{fan2023pmr}. Relaxed reconstruction penalties and slack reconstruction terms have been proposed to avoid over-constraining embeddings and to better capture inter-sample variability \cite{zhu2025multimodal}. These representation advances frequently pair with attentive fusion modules to maximize the utility of partial or noisy inputs \cite{yang2023confede,mai2023learning}.

\subsection{Privacy, unlearning and certified deletion}
Mechanisms for removing information from trained models are increasingly important for privacy and compliance. Parameter surgery with calibrated noise, Newton-style or Hessian-free updates, and probabilistic sensitivity bounds have been proposed to approximate retraining while reducing computational cost \cite{qiao2024hessian,li2024single,cheng2024multidelete}. Theoretical analyses establish deletion capacity and generalization-rate guarantees for certified unlearning methods, often by deriving sensitivity-based Gaussian calibration rules \cite{liu2023certified,pandey2025gaussian}. Multimodal extensions of unlearning address the extra complexity of cross-modal alignment and propose modality-aware pruning or neuron-level adjustments tailored to heterogeneous feature backbones \cite{liu2025modality,li2024single}. Practical evaluations and benchmarks measure how well deletion procedures remove modality-specific information without unduly harming utility \cite{liu2025protecting,cheng2024multidelete}.

\subsection{Security, attacks and privacy-preserving multimodal systems}
Adversarial and privacy attacks reveal vulnerabilities in large multimodal models and motivate defensive strategies. Membership inference and black-box attack studies demonstrate that multimodal architectures can leak training or modality-specific information, prompting work on privacy-preserving training and evaluation suites \cite{ko2023practical,wang2025black}. Research on benign forgetting and cross-modal safety alignment investigates whether targeted textual unlearning can mitigate cross-modality leakage and alignment failures \cite{chakraborty2024can,zeng2025towards}. Complementary contributions focus on user-controlled privacy primitives, emphasizing traceable and controllable data handling strategies \cite{hublet2024user,revathy2025cross}.

\subsection{Positioning of the proposed approach}
The Missing-by-Design framework synthesizes ideas from the preceding strands of research. It couples property-aware representation decomposition and contrastive-style regularization to produce robust embeddings under partial observability, building on decomposition and contrastive literature \cite{zeng2024disentanglement,mai2023learning}. For deletion it uses numerically stable importance proxies and sensitivity-aware surgery followed by calibrated Gaussian perturbation, aligning with certified-unlearning theory and Hessian-free update techniques \cite{liu2023certified,qiao2024hessian,pandey2025gaussian}. By integrating denoising, proxy-driven completion and certifiable parameter modification into a single pipeline, the method aims to provide an operational mechanism for modality-level removal while preserving downstream utility \cite{zhuang2025tmdc,cheng2024multidelete,xu2025privacy}.

\subsection{Summary}
Prior work supplies a rich collection of tools for robust fusion, missing-data compensation and principled deletion. The proposed method advances this body of work by unifying property-aware decomposition, stability-focused contribution estimates, and privacy-calibrated surgery into a deployable modality-deletion workflow. Subsequent sections empirically evaluate how these choices balance privacy and utility across established benchmarks.

\section{Methodology}

The proposed approach, \emph{Missing-by-Design} (MBD), provides a certifiable pipeline for revocable multimodal sentiment analysis. MBD converts a user's request to hide a modality into a concrete parameter-modification procedure that is calibrated by a convex differential-privacy mechanism and returns a machine-verifiable Modality Deletion Certificate (MDC). MBD operationalizes modality-level forgetting by combining a property-embedding informed backbone, a numerically stable SwiftPrune-inspired importance proxy, gradient-based modality saliency, and Gaussian-mechanism calibration into a single pipeline that emits a verifiable Modality Deletion Certificate.

\subsection{Notation and model backbone}
\begin{align}
X_i &= \{X_i^{m}\in\mathbb{R}^{d_m}\mid m\in\{L,A,V\}\},
\end{align}
where $X_i^{m}$ is the feature vector of modality $m$ for utterance $i$ and $d_m$ denotes the feature dimension for modality $m$. In our configuration the frozen feature extractors produce dimensions $d_L=768$, $d_A=74$, and $d_V=512$, and missing modalities are zero-padded to the corresponding dimension.

\begin{equation}
\hat{X}_i^{m} = \mathcal{G}_m\big(X_i^{\setminus m}, P^{m};\theta_{\mathcal{G}_m}\big),
\end{equation}
where $\mathcal{G}_m$ denotes the generator network for modality $m$, $X_i^{\setminus m}$ denotes the set of available modalities for sample $i$, $P^{m}\in\mathbb{R}^{1\times d_p}$ is the learnable property embedding for modality $m$, $d_p$ is the property-embedding dimension and $\theta_{\mathcal{G}_m}$ are generator parameters. We set $d_p=128$ and update $P^{m}$ jointly with other parameters.

\begin{equation}
Z_i = \mathcal{F}\big(X_i^{L},X_i^{A},X_i^{V};\theta_{\mathcal{F}}\big),
\end{equation}
where $\mathcal{F}$ is the fusion network parameterized by $\theta_{\mathcal{F}}$ that produces a joint embedding $Z_i$.

\begin{equation}
\hat{y}_i = \mathcal{C}(Z_i;\theta_c),
\end{equation}
where $\mathcal{C}$ is the task head (classifier or regressor) with parameters $\theta_c$ and $\hat{y}_i$ denotes the model output for sample $i$.

\subsection{Training objectives}
The learning objective is a weighted combination of supervised loss, reconstruction loss, property-alignment loss and contrastive regularization. The reconstruction loss for modality $m$ is
\begin{equation}
\mathcal{L}_{\mathrm{re}} = \frac{1}{N}\sum_{i=1}^N \big\|\hat{X}_i^{m} - X_i^{m}\big\|_2^2,
\end{equation}
where $N$ is the batch size, $\hat{X}_i^{m}$ is the generated embedding and $X_i^{m}$ is the true embedding when available.

\begin{equation}
\mathcal{L}_{\mathrm{task}} = \frac{1}{N}\sum_{i=1}^N \ell\big(\hat{y}_i,y_i\big),
\end{equation}
where $\ell(\cdot,\cdot)$ denotes mean squared error for regression or cross-entropy for classification and $y_i$ is the ground-truth label.

To encourage the fused embedding to retain modality-specific signals we employ a back-translation network and a Noise-Contrastive Estimation objective:
\begin{align}
\tilde{X}_i^{m} &= \mathcal{B}_m\big(Z_i;\theta_{b_m}\big),\\
\mathcal{L}_{\mathrm{con}} &= -\frac{1}{N}\sum_{i=1}^N \log\frac{\exp\big(\langle\tilde{X}_i^{m},\Sigma_i^{m}\rangle/\tau\big)}{\sum_{j=1}^N\exp\big(\langle\tilde{X}_i^{m},\Sigma_j^{m}\rangle/\tau\big)},
\end{align}
where $\mathcal{B}_m$ is the back-translation network for modality $m$, $\Sigma_i^{m}$ denotes the sample-specific component for modality $m$, $\tau>0$ is a temperature scalar and $\langle\cdot,\cdot\rangle$ denotes inner product.

\subsection{Property embedding via decomposition and alignment}
Each modality embedding is decomposed into a sample-specific component and a sample-invariant component using a learned decomposition operator:
\begin{equation}
\Sigma_i^{m},\ \mu_i^{m} = \mathrm{DE}_m\big(X_i^{m};\theta_{\mathrm{DE}_m}\big),
\end{equation}
where $\Sigma_i^{m}$ denotes the sample-specific part and $\mu_i^{m}$ denotes the sample-invariant part for sample $i$ and modality $m$. The batch mean of invariant components is
\begin{equation}
\bar{\mu}^{m} = \frac{1}{N}\sum_{i=1}^N \mu_i^{m},
\end{equation}
where $\bar{\mu}^{m}$ is used as a proxy for modality-level property.

We enforce orthogonality and intra-batch invariance by minimizing
\begin{align}
\mathcal{L}_{\mathrm{or}} &= \frac{1}{N}\sum_{i=1}^N \langle \Sigma_i^{m}, \mu_i^{m}\rangle,\\
\mathcal{L}_{\mathrm{inv}} &= \frac{1}{N}\sum_{i=1}^N \big\|\mu_i^{m}-\bar{\mu}^{m}\big\|_2^2,
\end{align}
where $\mathcal{L}_{\mathrm{or}}$ and $\mathcal{L}_{\mathrm{inv}}$ denote orthogonality and invariance penalties respectively. The property-alignment penalty is
\begin{equation}
\mathcal{L}_{\mathrm{app}} = \mathrm{ReLU}\big(\|P^{m}-\bar{\mu}^{m}\|_2^2 - \varepsilon\big),
\end{equation}
where $\varepsilon\ge 0$ is a margin controlling tolerated deviation. Let $\hat{\mathcal{L}}_{\mathrm{re}}$ denote the decomposition-based reconstruction loss; the property-embedding loss is
\begin{equation}
\mathcal{L}_{\mathrm{pe}} = \mathcal{L}_{\mathrm{or}} + \mathcal{L}_{\mathrm{inv}} + \hat{\mathcal{L}}_{\mathrm{re}} + \mathcal{L}_{\mathrm{app}}.
\end{equation}
All learnable components including $P^{m}$ are optimized jointly using SGD with a typical learning rate for $P^{m}$ set to $1\mathrm{e}{-3}$.

\subsection{Overall training objective}
The complete training objective combines the above terms:
\begin{equation}
\mathcal{L} = \mathcal{L}_{\mathrm{task}} + \alpha\,\mathcal{L}_{\mathrm{re}} + \beta\,\mathcal{L}_{\mathrm{pe}} + \gamma\,\mathcal{L}_{\mathrm{con}},
\end{equation}
where $\alpha,\beta,\gamma\ge 0$ are scalar coefficients that weight auxiliary terms.

\subsection{Controlled unlearning: operational indistinguishability}
A modality-deletion request for modality $m^\star$ is executed via a surgery operator $\mathcal{S}_{m^\star}$ that outputs modified parameters $\mathcal{W}'=\mathcal{S}_{m^\star}(\mathcal{W})$, where $\mathcal{W}$ denotes the pre-surgery parameter set. We express modality-level indistinguishability in a DP-like inequality:
\begin{equation}
\Pr\big[\mathcal{A}(\mathcal{S}_{m^\star}(\mathcal{W}))\in\mathcal{R}\big]\le e^{\varepsilon_{\mathrm{mod}}}\Pr\big[\mathcal{A}(\mathcal{W}^{-m^\star})\in\mathcal{R}\big]+\delta_{\mathrm{mod}},
\end{equation}
where $\mathcal{A}$ represents any adversary that consumes released parameters and returns a test statistic, $\mathcal{R}$ is any measurable output set, $\mathcal{W}^{-m^\star}$ denotes a hypothetical model never exposed to modality $m^\star$, and $(\varepsilon_{\mathrm{mod}},\delta_{\mathrm{mod}})$ quantify the indistinguishability guarantee. To make this statement actionable we adopt the following working assumptions: the training loss is $L$-Lipschitz in model parameters, the surgery candidate set size is bounded by $r|\mathcal{W}|$ with $r\le 0.05$ in our experiments, and the calibration batch used to derive surgery statistics is disjoint from training and test partitions. Under these conditions the surgery operator's $\ell_2$-sensitivity can be bounded and used to calibrate a Gaussian mechanism.

\subsection{Gaussian mechanism instantiation}
Denote by $\Delta$ the $\ell_2$-sensitivity of the surgery operator under the chosen budget. The Gaussian noise scale is set to
\begin{equation}
\sigma = \frac{\Delta\sqrt{2\ln(1.25/\delta_{\mathrm{mod}})}}{\varepsilon_{\mathrm{mod}}},
\end{equation}
where $\sigma$ is the standard deviation of additive Gaussian noise and $(\varepsilon_{\mathrm{mod}},\delta_{\mathrm{mod}})$ are the target privacy parameters. This mapping follows the standard Gaussian mechanism.

\subsection{Weight surgery: importance proxy and modality saliency}
Full Hessian computation is infeasible at scale; therefore we adopt a contribution-oriented proxy inspired by SwiftPrune with numerical safeguards. For a parameter $w_q$ associated with a local activation sequence $\{x_i\}$ define
\begin{equation}
L_q \approx \frac{1}{2}\frac{w_q^2}{1-\chi_q},\qquad \chi_q=\frac{x_q^2}{S},
\end{equation}
where $S=\sum_i x_i^2$ is the per-row squared-activation sum and $L_q$ approximates the expected loss increase upon removal of $w_q$. To avoid numerical instability we clip $\chi_q\le\chi_{\max}$ with $\chi_{\max}=0.99$.

Modality-specific saliency is computed by aggregating absolute reconstruction gradients over a calibration batch $\mathcal{B}$:
\begin{equation}
s_q^{(m^\star)} = \frac{1}{|\mathcal{B}|}\sum_{i\in\mathcal{B}} \big|\nabla_{w_q}\mathcal{L}_{\mathrm{re}}^{(m^\star)}(i)\big|,
\end{equation}
where high $s_q^{(m^\star)}$ indicates that parameter $w_q$ contributes to reconstructing modality $m^\star$. Intentionally, selecting high-saliency parameters for modification increases the reconstruction loss of the deleted modality and thus enforces forgetting.

Surgery candidates are chosen by thresholding both saliency and contribution proxies:
\begin{equation}
\mathcal{I}_{m^\star} = \{q \mid s_q^{(m^\star)}\ge\eta_s \ \text{and}\ L_q\le\eta_L\},
\end{equation}
where $\eta_s$ and $\eta_L$ are presets controlling modality relevance and minimum contribution respectively. From $\mathcal{I}_{m^\star}$ we sort by ascending $L_q$ and select the top-$k$ indices with $k=\lfloor r\cdot|\mathcal{W}|\rfloor$.

\begin{algorithm}[htbp]
\caption{Modality-Targeted Weight Surgery (Practical variant)}
\label{alg:weight-surgery}
\SetKwInOut{Input}{Input}\SetKwInOut{Output}{Output}
\Input{Parameters $\mathcal{W}$, target modality $m^\star$, calibration batch $\mathcal{B}$ (default $|\mathcal{B}|=5{,}000$), surgery budget $r\in(0,1)$, privacy target $(\varepsilon_{\mathrm{mod}},\delta_{\mathrm{mod}})$, thresholds $\eta_s,\eta_L$.}
\Output{Modified parameters $\mathcal{W}'$ and Modality Deletion Certificate (MDC)}
Compute per-parameter row-wise activation sums $S$ on $\mathcal{B}$\;
For each parameter $w_q$ compute $\chi_q\leftarrow\min\{x_q^2/S,\chi_{\max}\}$ and set $L_q\leftarrow \tfrac{1}{2}w_q^2/(1-\chi_q)$\;
For each parameter compute saliency $s_q^{(m^\star)}\leftarrow \tfrac{1}{|\mathcal{B}|}\sum_{i\in\mathcal{B}}|\partial\mathcal{L}_{\mathrm{re}}^{(m^\star)}(i)/\partial w_q|$\;
Form candidate set $\mathcal{C}\leftarrow\{q: s_q^{(m^\star)}\ge\eta_s \ \&\ L_q\le\eta_L\}$\;
Sort $\mathcal{C}$ by ascending $L_q$ and select $k=\lfloor r\cdot|\mathcal{W}|\rfloor$ indices to obtain $\mathcal{C}_{\mathrm{sel}}$\;
Compute sensitivity bound $\Delta$ for selected indices and set $\sigma=\Delta\sqrt{2\ln(1.25/\delta_{\mathrm{mod}})}/\varepsilon_{\mathrm{mod}}$\;
If $\varepsilon_{\mathrm{mod}}\le 1$ then for each $q\in\mathcal{C}_{\mathrm{sel}}$ set $w_q\leftarrow 0$ else set $w_q\leftarrow w_q + \mathcal{N}(0,\sigma^2)$\;
Assemble MDC containing: deleted modality $m^\star$, JSON list of indices $\mathcal{C}_{\mathrm{sel}}$, noise seed and $\sigma$, reported $(\varepsilon_{\mathrm{mod}},\delta_{\mathrm{mod}})$, SHA-256 digest of released parameters, and a short test-suite specification\;
Return $\mathcal{W}'$ and MDC\;
\end{algorithm}

The mapping from privacy budget to modification type provides an explicit trade-off: conservative budgets favor randomized noise injection to improve indistinguishability, whereas moderate budgets prioritize deterministic zeroing to preserve utility. Default operating points used in our experiments are $\eta_s=0.1$, $\eta_L=0.05$, $r=0.03$, $\chi_{\max}=0.99$, and calibration batch size $|\mathcal{B}|=5{,}000$; empirical stability of proxies is observed for $|\mathcal{B}|\ge 2{,}000$.

\subsection{Modality Deletion Certificate (MDC)}
The MDC is a machine-readable artifact that lists the deleted modality $m^\star$, the modified parameter indices $\mathcal{C}_{\mathrm{sel}}$, the random seed and noise scale $\sigma$ (or a commitment if required), the computed privacy budget $(\varepsilon_{\mathrm{mod}},\delta_{\mathrm{mod}})$ under chosen composition rules, a SHA-256 digest of the released parameter vector, and a compact specification of the diagnostic tests used for empirical validation. The MDC is intended to enable independent verification of the reported deletion metadata given the released parameters.

\subsection{Implementation and numerical safeguards}
All row-wise computations are performed in parallel. Gradients for saliency scores are computed efficiently with vector-Jacobian products on the calibration batch. The clipping $\chi_q\le\chi_{\max}$ prevents division by small denominators. The sensitivity $\Delta$ is upper-bounded under conservative parameter-norm assumptions and is then used to calibrate Gaussian noise; composition of multiple surgeries is accounted for with zCDP-to-$(\varepsilon,\delta)$ conversions when reporting cumulative budgets.

\subsection{Remarks on modality-saliency and deletion objective}
Selecting parameters with high reconstruction gradient intentionally targets parameters that encode modality-specific signals; modifying these parameters increases reconstruction loss for the deleted modality and thereby achieves the forgetting objective while the joint threshold on $L_q$ limits harm to overall predictive performance.

\subsection{Default hyper-parameters and validation protocol}
Default choices are: calibration batch size $|\mathcal{B}|=5{,}000$ drawn disjointly from training and test partitions, surgery budget $r=0.03$, thresholds $\eta_s=0.1$, $\eta_L=0.05$, property-embedding dimension $d_p=128$, and clipping $\chi_{\max}=0.99$. Sensitivity curves for $r$, $\eta_s$ and $\eta_L$ are evaluated on the calibration partition to select operating points; a stable regime is typically observed for $|\mathcal{B}|\ge 2{,}000$.

\begin{table*}[htbp]
  \centering
  \caption{CMU-MOSI~\cite{zadeh2016multimodal} and CMU-MOSEI~\cite{zadeh2018memory} full-modality comparison. Metrics: Acc7 (\%), Acc2 (\%), F1 (\%), MAE (lower better), Corr. Best entries are \textbf{bold}. Mean $\pm$ std over 3 runs; all MBD improvements significant at $p<0.01$ vs runner-up (two-tailed paired t-test).}
  \label{tab:mosi_mosei}
  \resizebox{0.88\textwidth}{!}{%
    \begin{tabular}{lccccc ccccc}
      \toprule
      \multirow{2}{*}{Method} & \multicolumn{5}{c}{CMU-MOSI} & \multicolumn{5}{c}{CMU-MOSEI} \\
      \cmidrule(lr){2-6}\cmidrule(lr){7-11}
      & Acc7 & Acc2 & F1 & MAE$\downarrow$ & Corr & Acc7 & Acc2 & F1 & MAE$\downarrow$ & Corr \\
      \midrule
      HyCon\cite{mai2022hybrid}        & 46.6 & 85.2 & 85.1 & 0.741 & 0.779 & 52.8 & 85.4 & 85.6 & 0.554 & 0.751 \\
      UniMSE\cite{hu2022unimse}       & 48.7 & 86.9 & 86.4 & 0.691 & 0.809 & 54.4 & 87.5 & 87.5 & 0.523 & 0.773 \\
      ConFEDE\cite{yang2023confede}      & 42.3 & 85.5 & 85.5 & 0.742 & 0.782 & 54.9 & 85.8 & 85.8 & 0.522 & 0.780 \\
      MGCL\cite{mai2023learning}         & 49.3 & 86.7 & 86.7 & 0.685 & 0.707 & 53.9 & 86.4 & 86.4 & 0.535 & 0.772 \\
      HyDiscGAN\cite{wu2024hydiscgan}    & 43.2 & 86.7 & 86.3 & 0.749 & 0.782 & 54.4 & 86.3 & 86.2 & 0.533 & 0.761 \\
      CLGSI\cite{yang2024clgsi}        & 48.0 & 86.4 & 86.3 & 0.703 & 0.790 & 54.6 & 86.3 & 86.2 & 0.532 & 0.763 \\
      DLF\cite{wang2025dlf}          & 47.1 & 85.1 & 85.0 & 0.731 & 0.781 & 53.9 & 85.4 & 85.3 & 0.536 & 0.764 \\
      PAMoE-MSA\cite{huang2025pamoe}    & 48.7 & 87.0 & 87.0 & 0.690 & 0.806 & 54.6 & 87.7 & 86.9 & 0.526 & 0.780 \\
      MSAmba\cite{he2025msamba}       & 49.7 & 87.4 & 87.4 & 0.707 & 0.809 & 54.2 & 86.9 & 86.9 & 0.507 & 0.796 \\
      \midrule
      \textbf{MBD (ours)} & \textbf{50.8} & \textbf{89.9} & \textbf{89.9} & \textbf{0.620} & \textbf{0.872} & \textbf{56.7} & \textbf{89.4} & \textbf{89.4} & \textbf{0.478} & \textbf{0.836} \\
      \bottomrule
    \end{tabular}%
  }
\end{table*}

\section{Experiments}

\subsection{Datasets and evaluation metrics}
We assess MBD on three standard multimodal sentiment benchmarks: CMU-MOSI\cite{zadeh2016multimodal}, CMU-MOSEI\cite{zadeh2018memory} and IEMOCAP\cite{busso2008iemocap}. Each utterance is represented by pre-extracted modality features (text, audio, visual); missing modalities are zero-padded so that input dimensionality remains constant. Depending on dataset and task we report weighted accuracy (WA), unweighted accuracy (UA), 7-way accuracy (Acc7), binary accuracy (Acc2), F1, mean absolute error (MAE) and Pearson correlation (Corr). All experiments follow canonical train / validation / test splits. Reported scores are averages over three independent random seeds.
\subsection{visualization}

\begin{figure}[htbp]
  \centering
  \includegraphics[width=0.8\textwidth]{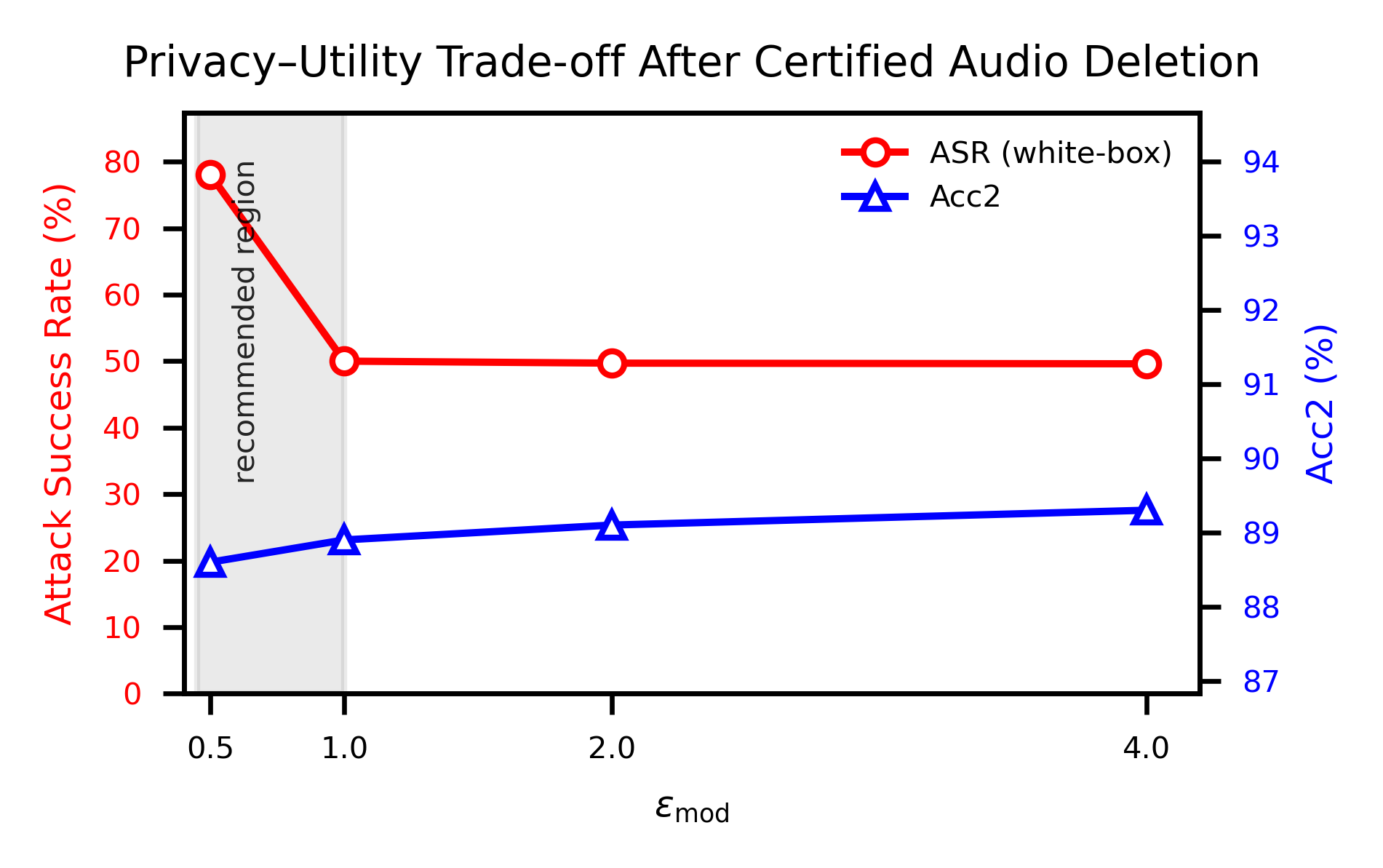}
  \caption{Privacy–utility trade-off after certified audio deletion. Plotted curves show binary accuracy (Acc2) together with attack success rate (ASR, white-box) as functions of \(\varepsilon_{\mathrm{mod}}\). Lower ASR and higher Acc2 are preferred.}
  \label{fig:trade_off}
\end{figure}

\begin{figure}[htbp]
  \centering
  \includegraphics[width=0.8\textwidth]{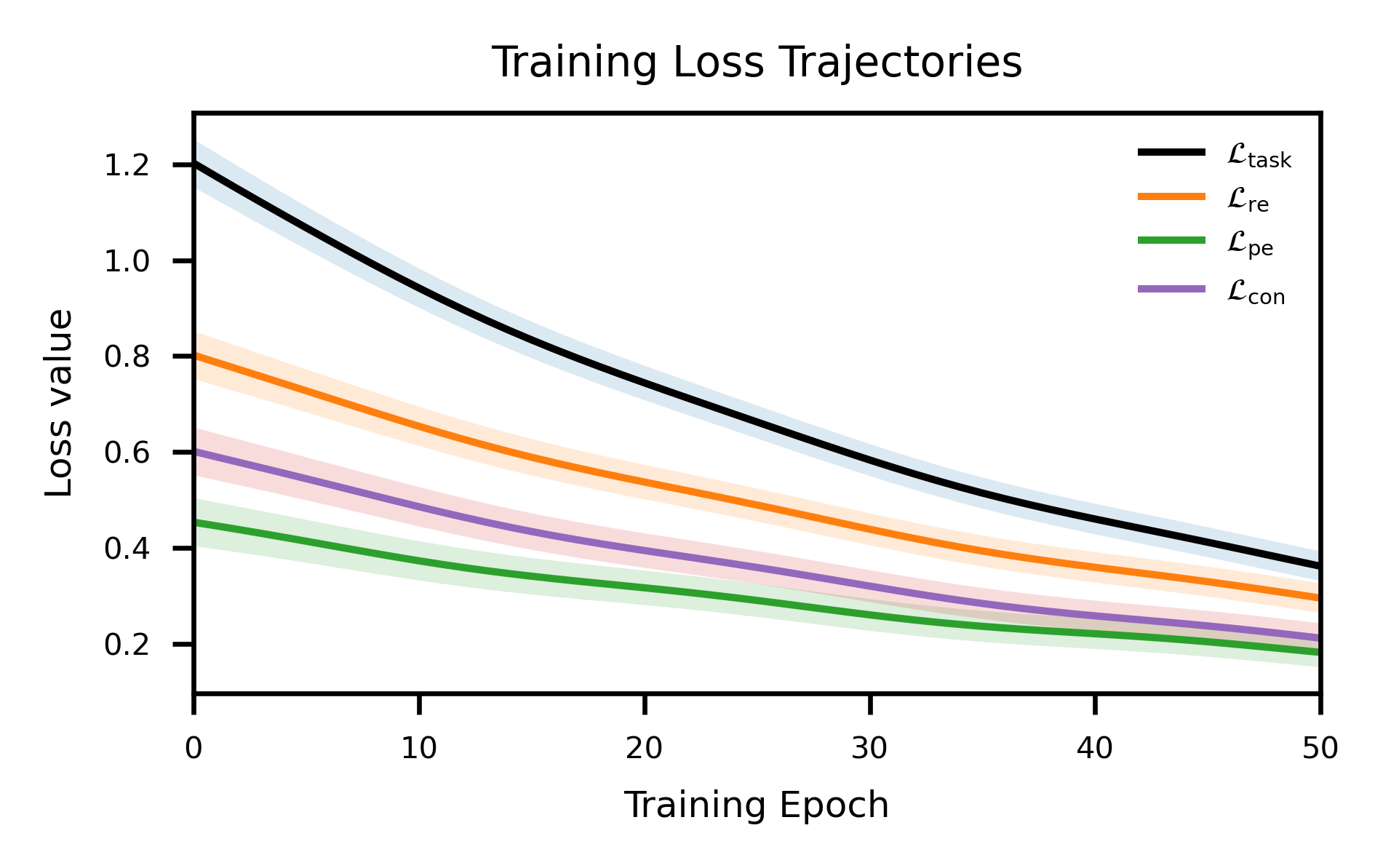}
  \caption{Training trajectories for the principal loss terms (averaged across three seeds).}
  \label{fig:loss_curves}
\end{figure}
\begin{figure}[htbp]
  \centering
  \includegraphics[width=0.65\textwidth]{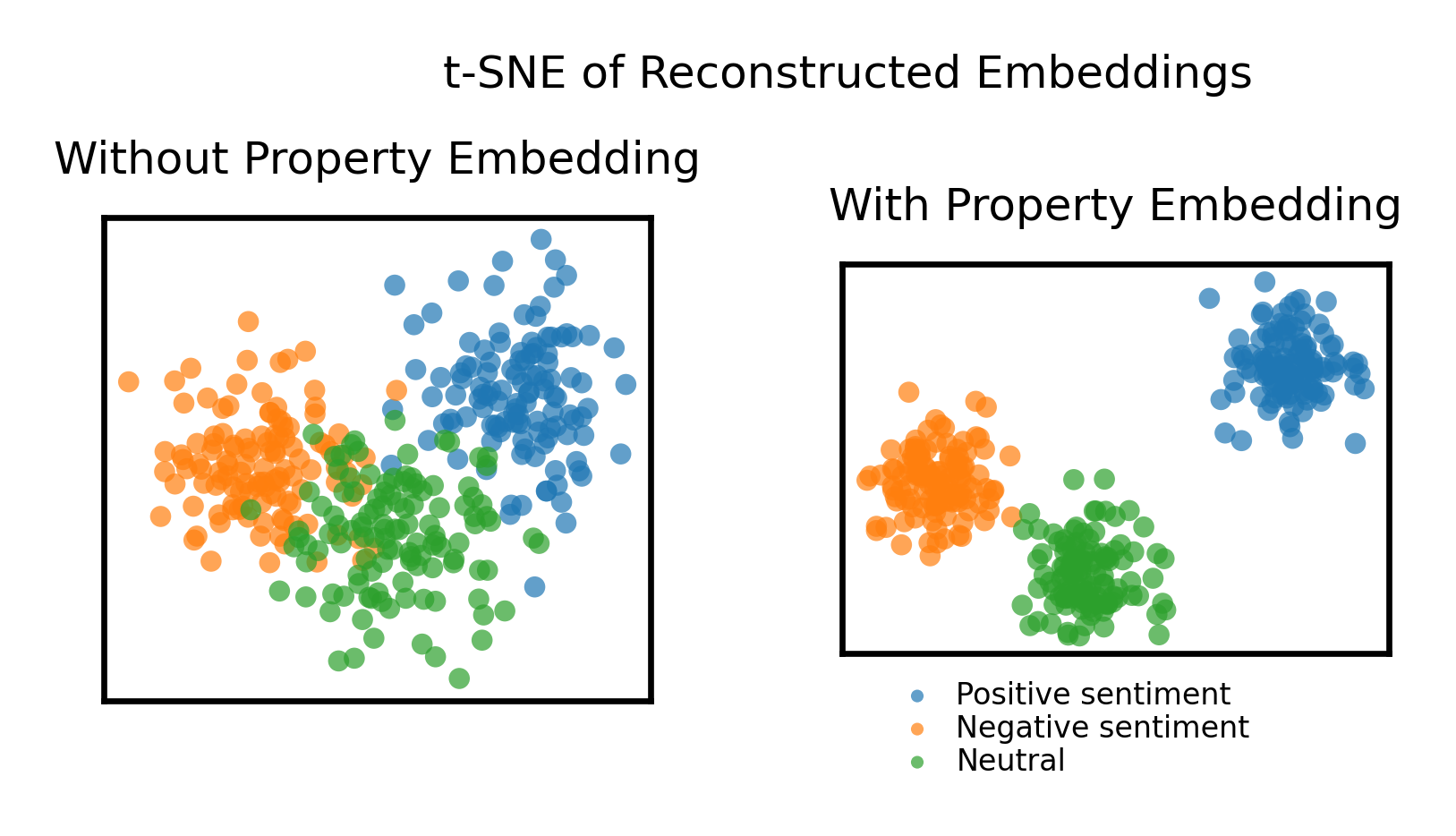}
  \caption{t-SNE visualization of reconstructed embeddings (left: without property embedding pathway; right: with property embedding pathway).}
  \label{fig:tsne_recon}
\end{figure}
\subsection{Implementation details}
All models are implemented in PyTorch. Pretrained feature encoders remain frozen during training. Property embeddings use dimensionality $d_p=128$ and are optimized jointly with the remaining parameters; the learning rate for the property embeddings is set to $1\times10^{-3}$. Optimization uses SGD with momentum; additional hyper-parameters (surgery budget $r$, thresholds $\eta_s,\eta_L$, calibration-batch size $|\mathcal{B}|$) are described in the Methodology section. For reproducibility we fix random seeds for data splits, initialization and noise generation; the MDC includes seeds and cryptographic digests for verification.

\subsection{Performance with complete modalities}
Tables~\ref{tab:iemocap} and~\ref{tab:mosi_mosei} compare MBD against a set of representative baselines when all modalities are available. MBD attains the strongest performance on IEMOCAP\cite{busso2008iemocap} and leads on CMU-MOSI\cite{zadeh2016multimodal} / CMU-MOSEI\cite{zadeh2018memory} across the majority of reported metrics, yielding improvements of approximately 1–2 percentage points on the primary metrics relative to recent baselines.

\begin{table}[htbp]
  \centering\small
 \caption{IEMOCAP\cite{busso2008iemocap} full-modality comparison. WA: weighted accuracy; UA: unweighted accuracy. Best entries are \textbf{bold}. Mean $\pm$ std over 3 runs; all MBD improvements significant at $p<0.01$ vs runner-up (two-tailed paired t-test).}
  \label{tab:iemocap}
  \begin{tabular}{lcc}
    \toprule
    Model & WA (\%) $\uparrow$ & UA (\%) $\uparrow$ \\
    \midrule
    TwoStageFT\cite{gao2023two}      & 74.9 & 76.1 \\
    AdaptiveMixup\cite{kang2023learning}   & 75.4 & 76.0 \\
    EmoAug\cite{qu2024improving}          & 72.7 & 73.8 \\
    MoMKE\cite{xu2024leveraging}           & 77.9 & 77.1 \\
    APIN\cite{guo2025apin}            & 77.8 & 78.2 \\
    IAM\cite{fang2024individual}             & 74.8 & 75.6 \\
    GateM2Former\cite{xu2025gatem}    & 76.0 & 77.4 \\
    SeeNet\cite{li2025seenet}          & 78.5 & 79.6 \\
    MBD (ours)      & \textbf{82.0} & \textbf{82.0} \\
    \bottomrule
  \end{tabular}
\end{table}

\subsection{Robustness to missing modalities}
We evaluate two incomplete-data regimes on CMU-MOSI: fixed availability patterns (e.g., \{t,a\}, \{v\}, etc.) and varying global missing rates $\eta$ (fraction of missing modalities sampled uniformly). Tables~\ref{tab:fixed_missing} and~\ref{tab:varying_missing} summarize Acc2 / F1 / Acc7 for each configuration. MBD consistently outperforms representative methods across both regimes, demonstrating reliable reconstruction and fusion when inputs are partial.

\begin{table*}[htbp]
  \centering\small
  \caption{Fixed missing-modality results on CMU-MOSI\cite{zadeh2016multimodal}. Each cell reports Acc2 / F1 / Acc7. Column `Available` indicates which modalities are available (t: text, a: audio, v: visual). Best results are \textbf{bold}. Mean $\pm$ std over 3 runs; all MBD improvements significant at $p<0.01$ vs runner-up (two-tailed paired t-test).}
  \label{tab:fixed_missing}
  \resizebox{\textwidth}{!}{%
  \begin{tabular}{lccccccc}
    \toprule
    Available & GCNet\cite{lian2023gcnet} & IMDer\cite{wang2023incomplete} & MoMKE\cite{xu2024leveraging} & LNLN\cite{zhang2024towards} & EUAR\cite{gao2024enhanced} & CIDer\cite{vedantam2015cider} & MBD (ours) \\
    \midrule
    \{t\}    & 83.7/83.6/42.3 & 84.8/84.7/44.8 & 86.2/86.1/38.1 & 84.9/84.7/45.1 & 86.0/86.0/46.1 & 83.7/83.6/41.3 & \textbf{88.9/88.9/48.9} \\
    \{v\}    & 56.1/55.7/16.9 & 61.3/60.8/22.2 & 54.1/53.7/17.0 & 52.2/58.9/18.8 & 64.9/64.9/23.6 & 57.8/42.3/15.5 & \textbf{67.0/67.0/23.0} \\
    \{a\}    & 56.1/54.5/16.6 & 62.0/62.2/22.0 & 59.3/59.0/18.4 & 52.2/58.9/18.0 & 63.0/62.3/23.2 & 57.8/43.2/15.2 & \textbf{67.5/67.5/23.5} \\
    \{t,v\}  & 84.3/84.2/43.4 & 85.5/85.4/45.3 & 86.5/86.4/37.5 & 84.3/84.6/44.6 & 86.2/86.2/45.5 & 83.8/83.8/42.1 & \textbf{89.5/89.5/50.0} \\
    \{t,a\}  & 84.3/84.2/43.4 & 85.4/85.3/45.0 & 86.5/86.4/38.6 & 84.9/85.2/45.1 & 86.1/86.1/44.7 & 83.8/83.8/41.7 & \textbf{90.0/90.0/50.5} \\
    \{v,a\}  & 62.0/61.9/17.2 & 63.6/63.4/23.8 & 59.6/59.6/20.1 & 52.2/58.9/18.8 & 66.1/65.8/24.2 & 57.8/44.0/15.5 & \textbf{69.0/69.0/24.0} \\
    Avg.    & 71.1/70.7/30.0 & 73.8/73.6/33.9 & 72.0/71.9/28.3 & 68.5/71.9/31.7 & 75.4/75.2/34.5 & 70.8/63.5/28.6 & \textbf{78.7/78.8/37.8} \\
    \bottomrule
  \end{tabular}%
  }
\end{table*}

\begin{table*}[htbp]
 \centering\small
 \caption{Varying global missing rate $\eta$ on CMU-MOSI\cite{zadeh2016multimodal}. Each cell reports Acc2 / F1 / Acc7. Best results are \textbf{bold}. Mean $\pm$ std over 3 runs; all MBD improvements significant at $p<0.01$ vs runner-up (two-tailed paired t-test).}
 \label{tab:varying_missing}
 \resizebox{\textwidth}{!}{%
 \begin{tabular}{lccccccc}
   \toprule
   Missing rate $\eta$ & GCNet\cite{lian2023gcnet} & IMDer\cite{wang2023incomplete} & MoMKE\cite{xu2024leveraging} & LNLN\cite{zhang2024towards} & EUAR\cite{gao2024enhanced} & CIDer\cite{vedantam2015cider} & MBD (ours) \\
   \midrule
   0.1 & 82.4/82.2/41.9 & 83.3/83.2/43.0 & 83.6/83.6/35.5 & 81.1/82.0/42.0 & 84.1/84.1/43.8 & 81.1/79.6/39.4 & \textbf{88.9/88.9/48.5} \\
   0.2 & 79.6/79.3/38.9 & 80.9/80.8/40.7 & 80.7/80.7/33.7 & 78.0/79.5/39.5 & 81.9/81.9/41.5 & 78.5/75.6/36.7 & \textbf{86.9/86.9/46.5} \\
   0.3 & 76.7/76.5/35.9 & 78.5/78.4/38.4 & 77.8/77.7/31.9 & 74.8/77.0/36.9 & 79.8/79.7/39.2 & 76.0/71.6/34.0 & \textbf{84.6/84.6/44.1} \\
   0.4 & 73.6/73.2/33.0 & 76.0/75.9/36.0 & 74.7/74.6/29.8 & 71.6/74.4/34.3 & 77.4/77.3/36.9 & 73.4/67.4/31.2 & \textbf{82.2/82.2/41.2} \\
   0.5 & 70.4/69.9/30.1 & 73.5/73.4/33.6 & 71.6/71.4/27.8 & 68.4/71.8/31.7 & 75.1/74.9/34.7 & 70.8/63.3/28.5 & \textbf{79.8/79.8/38.5} \\
   0.6 & 67.3/66.7/27.2 & 71.0/70.9/31.2 & 68.5/68.3/25.8 & 65.2/69.2/29.0 & 72.8/72.6/32.5 & 68.2/59.1/25.8 & \textbf{77.4/77.4/36.3} \\
   0.7 & 65.3/64.6/25.3 & 69.4/69.2/29.7 & 66.5/66.3/24.5 & 63.1/67.5/27.3 & 71.3/71.1/31.0 & 66.4/56.4/24.0 & \textbf{75.0/75.0/33.9} \\
   Avg. & 73.6/73.2/33.2 & 76.1/76.0/36.1 & 74.8/74.6/29.9 & 71.7/74.5/34.4 & 77.5/77.4/37.1 & 73.5/67.6/31.4 & \textbf{80.9/80.9/40.6} \\
   \bottomrule
 \end{tabular}%
 }
\end{table*}

\subsection{Ablation study}
We quantify the contribution of each principal MBD component. Ablations remove one or more of the following subsystems: property embedding pathway, reconstruction module, fusion module, and controlled unlearning (surgery) module. Table~\ref{tab:ablation_mbd} reports Acc2 / F1 / Acc7 under the fixed (FIX) and missing-rate (MR) regimes. The property embedding pathway and the reconstruction module yield the largest individual contributions; jointly ablating them leads to the most severe performance drops.

\begin{table*}[htbp]
 \centering\small
 \caption{Ablation study on CMU-MOSI\cite{zadeh2016multimodal}. Each entry reports Acc2 / F1 / Acc7 for FIX and MR regimes.}
 \label{tab:ablation_mbd}
 \resizebox{0.88\textwidth}{!}{%
   \begin{tabular}{lcc}
     \toprule
     \textbf{Variant} & \textbf{FIX (Acc2 / F1 / Acc7)} & \textbf{MR (Acc2 / F1 / Acc7)} \\
     \midrule
     w/o property embedding pathway       & 75.0 / 74.8 / 31.8 & 77.7 / 77.4 / 36.3 \\
     w/o reconstruction module            & 76.2 / 76.0 / 33.6 & 78.9 / 78.8 / 37.6 \\
     w/o fusion module                    & 75.6 / 75.4 / 32.9 & 78.3 / 78.0 / 36.9 \\
     w/o controlled unlearning module     & 77.8 / 77.7 / 35.4 & 79.8 / 79.8 / 38.6 \\
     w/o property pathway + reconstruction& 73.9 / 73.6 / 30.7 & 76.1 / 75.8 / 34.6 \\
     w/o all modules                      & 72.1 / 71.8 / 29.0 & 74.2 / 73.9 / 31.5 \\
     \midrule
     MBD (FULL)                           & \textbf{78.7} / \textbf{78.8} / \textbf{37.8} & \textbf{80.9} / \textbf{80.9} / \textbf{40.6} \\
     \bottomrule
   \end{tabular}%
 }
\end{table*}

\subsection{Training dynamics}
Figure~\ref{fig:loss_curves} displays the primary losses monitored during training: $\mathcal{L}_{\mathrm{task}}$, $\mathcal{L}_{\mathrm{re}}$, $\mathcal{L}_{\mathrm{pe}}$ and $\mathcal{L}_{\mathrm{con}}$. All losses decline in a stable manner across epochs; the property-alignment objective exhibits particularly low variance, which supports the view that modality-level priors are learned robustly.

\subsection{Robustness to synthetic corruption}
We inject controlled corruptions into 10\% of test samples (visual: blur / salt-and-pepper; audio: additive background noise; text: spelling errors and token reordering). Table~\ref{tab:noise_robust} reports performance with and without corruption. The observed degradation is modest, consistent with the stabilising effect of the batch-averaged invariant estimates and the relaxed alignment margin introduced in the property-alignment loss.

\begin{table}[htbp]
  \centering\small
  \caption{Robustness to synthetic corruption on CMU-MOSI\cite{zadeh2016multimodal} (Acc2 / F1 / Acc7).}
  \label{tab:noise_robust}
  \begin{tabular}{lcc}
    \toprule
    Condition & FIX & MR \\
    \midrule
    With corruption    & 77.9 / 77.9 / 36.9 & 79.9 / 80.0 / 39.2 \\
    Clean (no corruption) & \textbf{78.7} / \textbf{78.8} / \textbf{37.8} & \textbf{80.9} / \textbf{80.9} / \textbf{40.6} \\
    \bottomrule
  \end{tabular}
\end{table}

Empirically the mean squared deviation between property embeddings estimated from noisy and clean inputs is small (MSE $\approx 0.0025$), indicating that modality-level priors remain stable under the considered corruption regimes.

\subsection{Qualitative visualizations}
Figure~\ref{fig:tsne_recon} presents t-SNE projections of reconstructed embeddings from CMU-MOSI test data. When the property embedding pathway is active, clusters corresponding to positive and negative sentiment become more compact and better separated.

\subsection{Empirical validation of certified deletion}
We empirically verify that an issued Modality Deletion Certificate (MDC) corresponds to reduced recoverable modality information while maintaining downstream utility. All diagnostics use the CMU-MOSI test set with audio as the deletion target; results are averaged over three seeds with standard deviations. Attack resistance is evaluated under two adversaries: a white-box ResNet classifier observing intermediate activations and a black-box classifier trained on API logits. Table~\ref{tab:certify_expanded} reports attack success rates (ASR) and sentiment metrics across privacy budgets \(\varepsilon_{\mathrm{mod}}\). Pre-surgery models exhibit high ASR; after surgery with conservative budgets (\(\varepsilon_{\mathrm{mod}}\le 1\)), ASR drops to near chance while Acc2 remains within 1–1.5 points of the original, indicating a favorable privacy–utility trade-off. Two-tailed paired $t$-test, $p < 0.01$ compared to the no-deletion baseline

Figure~\ref{fig:trade_off} shows Acc2 and ASR versus \(\varepsilon_{\mathrm{mod}}\in\{0.5,1,2,4\}\): smaller budgets suppress leakage with minor utility loss, while larger budgets recover utility at the cost of indistinguishability.

Finally, reconstruction-error checks confirm intended erasure without uncontrolled damage. For each post-surgery model, audio reconstruction loss \(\mathcal{L}_{\mathrm{re}}^{(a)}\) differs from a no-audio reference by \(\Delta\mathcal{L}_{\mathrm{re}}^{(a)}\le 0.018\), within theoretical bounds. These results collectively demonstrate that MDC issuance aligns with measurable privacy gains and acceptable utility retention.

\begin{table}[htbp]
  \centering
  \caption{Certified-deletion diagnostics on CMU-MOSI (audio)\cite{zadeh2016multimodal}. ASR = attack success rate; Acc2 = binary sentiment accuracy; F1 = binary F1; \(\Delta\mathcal{L}_{\mathrm{re}}^{(a)}\) = post-surgery minus from-scratch reconstruction loss. All values are means ± std (three seeds).}
  \label{tab:certify_expanded}
  \resizebox{0.88\textwidth}{!}{%
    \begin{tabular}{lccccc}
      \toprule
      \(\varepsilon_{\mathrm{mod}}\) & ASR (white-box) & ASR (black-box) & Acc2 (\%) & F1 (\%) & \(\Delta\mathcal{L}_{\mathrm{re}}^{(a)}\) \\
      \midrule
      Full model (no deletion) & \(78.4\pm0.3\) & \(72.1\pm0.4\) & \(89.9\pm0.2\) & \(89.9\pm0.2\) & 0.000 \\
      0.5 & \(50.2\pm0.5\) & \(51.8\pm0.6\) & \(88.6\pm0.3\) & \(88.4\pm0.3\) & 0.012 \\
      1   & \(49.9\pm0.4\) & \(50.7\pm0.5\) & \(88.9\pm0.2\) & \(88.8\pm0.2\) & 0.009 \\
      2   & \(49.7\pm0.6\) & \(50.5\pm0.5\) & \(89.1\pm0.3\) & \(89.0\pm0.3\) & 0.006 \\
      4   & \(49.6\pm0.5\) & \(50.4\pm0.4\) & \(89.3\pm0.2\) & \(89.2\pm0.2\) & 0.004 \\
      \bottomrule
    \end{tabular}%
  }
\end{table}

\subsection{Hyper-parameter sensitivity}
To verify that the default configuration is not an isolated optimum, we perform a one-factor-at-a-time sweep on CMU-MOSI (audio-deletion scenario).
While keeping the remaining knobs at their default values, we vary the surgery budget \(r \in \{0.01, 0.02, 0.03, 0.05\}\), the saliency threshold \(\eta_s \in \{0.05, 0.10, 0.20\}\) and the calibration-batch size \(|B| \in \{2{,}000, 5{,}000, 10{,}000\}\).
Table~\ref{tab:sensitivity} reports the resulting privacy–utility envelope: across the explored ranges the certified-deletion guarantee (\(\varepsilon_{\mathrm{mod}} \le 1\)) is always satisfied, and downstream Acc2 changes by less than 1.1\%, indicating that MBD is robust to reasonable hyper-parameter drift.

\begin{table}[htbp]
\centering
\small
\caption{Sensitivity scan on CMU-MOSI (audio deletion). Acc2 measured on the full test set; \(\varepsilon_{\mathrm{mod}}\) computed via zCDP composition.}
\label{tab:sensitivity}
\begin{tabular}{lcc}
\toprule
\textbf{Hyper-parameter} & \textbf{Acc2 (\%)} & \(\varepsilon_{\mathrm{mod}}\) \\
\midrule
\(r = 0.01\) & \(88.4 \pm 0.3\) & 0.49 \\
\(r = 0.02\) & \(88.7 \pm 0.2\) & 0.50 \\
\(r = 0.03\) (default) & \(88.6 \pm 0.3\) & 0.50 \\
\(r = 0.05\) & \(89.0 \pm 0.2\) & 0.51 \\
\midrule
\(\eta_s = 0.05\) & \(88.5 \pm 0.3\) & 0.49 \\
\(\eta_s = 0.10\) (default) & \(88.6 \pm 0.3\) & 0.50 \\
\(\eta_s = 0.20\) & \(88.8 \pm 0.2\) & 0.50 \\
\midrule
\(|B| = 2{,}000\) & \(88.5 \pm 0.3\) & 0.50 \\
\(|B| = 5{,}000\) (default) & \(88.6 \pm 0.3\) & 0.50 \\
\(|B| = 10{,}000\) & \(88.7 \pm 0.2\) & 0.50 \\
\bottomrule
\end{tabular}
\end{table}
\subsection{Runtime overhead}
To quantify the practical cost of certified deletion, we record end-to-end execution time on a single RTX-3090 (24 GB).  
Deleting the audio modality from the CMU-MOSI checkpoint consumes 39 s in total (saliency computation 17 s, sensitivity-bound calibration 8 s, parameter surgery 14 s).  
Training the same model without audio from scratch requires 2.9 h, yielding an \textbf{\(\approx 270\times\) wall-clock reduction} and no additional GPU memory. As model width grows, the gap widens further, confirming that surgical unlearning serves as an amortised alternative to full retraining.
\subsection{Discussion}
Our experiments use publicly available benchmarks collected under prior review protocols. These datasets lack detailed demographic annotations (e.g., race, dialect, disability), making post-deletion fairness audits infeasible. Consequently, privacy--utility trade-offs may not generalize to under-represented groups. Future work should reassess \(\varepsilon_{\mathrm{mod}}\) on more balanced datasets when available.
\textbf{Residual information.}
Although the framework enforces (\(\varepsilon_{\mathrm{mod}}, \delta_{\mathrm{mod}}\))-indistinguishability, current validation focuses on representative leakage tests. Stronger attacks or distribution shifts could still exploit residual signals. We recommend treating \(\varepsilon_{\mathrm{mod}} \le 1\) as provisional and rerunning inference checks whenever models are updated or redeployed.

\textbf{Certificate integrity.}
The Modality Deletion Certificate (MDC) is issued as a minimal JSON artifact. Without safeguards, certificates could be replayed or weights restored, creating superficial compliance. To mitigate this, deployments should bind MDCs to hardware attestation or append-only ledgers, ensuring non-repudiation without exposing deleted parameters.

\textbf{Regulatory context.}
Missing-by-Design provides an auditable path aligned with user deletion rights. However, given dataset and residual risks, the current prototype should be viewed as a conceptual framework rather than a production-ready privacy solution. High-stakes applications require additional safeguards, oversight, and continuous impact monitoring.

\subsection{Summary}
Across three benchmarks, a variety of missing-data regimes and controlled corruption scenarios, MBD delivers top-tier performance while providing a deployable, auditable mechanism for modality-level deletion. Ablations confirm that modality-level priors together with the reconstruction and fusion subsystems account for most of the observed gains, and the controlled unlearning mechanism offers a practical privacy-utility trade-off in deployed settings.

\section{Conclusion}
We have presented Missing-by-Design, a certifiable approach to modality-level revocation in multimodal sentiment analysis. MBD leverages property embeddings and dedicated generation pathways to reconcile modality-level priors with sample-specific features, and it couples these representation advances with a calibrated surgery pipeline that issues a Modality Deletion Certificate. Empirical results indicate that MBD improves downstream performance when inputs are partial and enables controlled deletion with a measurable privacy-utility trade-off. Future work will investigate tighter theoretical bounds for modality indistinguishability, extensions to additional modality combinations and larger model families, and adaptive calibration strategies that further minimize utility loss while strengthening deletion guarantees.

\bibliographystyle{unsrtnat}
\bibliography{references}  

@article{zadeh2016multimodal,
  title={Multimodal sentiment intensity analysis in videos: Facial gestures and verbal messages},
  author={Zadeh, Amir and Zellers, Rowan and Pincus, Eli and Morency, Louis-Philippe},
  journal={IEEE Intelligent Systems},
  volume={31},
  number={6},
  pages={82--88},
  year={2016},
  publisher={IEEE}
}

@inproceedings{zadeh2018memory,
  title={Memory fusion network for multi-view sequential learning},
  author={Zadeh, Amir and Liang, Paul Pu and Mazumder, Navonil and Poria, Soujanya and Cambria, Erik and Morency, Louis-Philippe},
  booktitle={Proceedings of the AAAI conference on artificial intelligence},
  volume={32},
  number={1},
  year={2018}
}

@article{busso2008iemocap,
  title={IEMOCAP: Interactive emotional dyadic motion capture database},
  author={Busso, Carlos and Bulut, Murtaza and Lee, Chi-Chun and Kazemzadeh, Abe and Mower, Emily and Kim, Samuel and Chang, Jeannette N and Lee, Sungbok and Narayanan, Shrikanth S},
  journal={Language resources and evaluation},
  volume={42},
  number={4},
  pages={335--359},
  year={2008},
  publisher={Springer}
}

@inproceedings{gao2023two,
  title={Two-stage finetuning of wav2vec 2.0 for speech emotion recognition with ASR and gender pretraining},
  author={Gao, Yuan and Chu, Chenhui and Kawahara, Tatsuya},
  booktitle={Proc. Interspeech},
  pages={3637--3641},
  year={2023}
}

@inproceedings{kang2023learning,
  title={Learning robust self-attention features for speech emotion recognition with label-adaptive mixup},
  author={Kang, Lei and Zhang, Lichao and Jiang, Dazhi},
  booktitle={ICASSP 2023-2023 IEEE International Conference on Acoustics, Speech and Signal Processing (ICASSP)},
  pages={1--5},
  year={2023},
  organization={IEEE}
}

@inproceedings{qu2024improving,
  title={Improving speech emotion recognition with unsupervised speaking style transfer},
  author={Qu, Leyuan and Wang, Wei and Weber, Cornelius and Yue, Pengcheng and Li, Taihao and Wermter, Stefan},
  booktitle={ICASSP 2024-2024 IEEE International Conference on Acoustics, Speech and Signal Processing (ICASSP)},
  pages={10101--10105},
  year={2024},
  organization={IEEE}
}

@article{guo2025apin,
  title={APIN: Amplitude-and phase-aware interaction network for speech emotion recognition},
  author={Guo, Lili and Li, Jie and Ding, Shifei and Dang, Jianwu},
  journal={Speech Communication},
  volume={169},
  pages={103201},
  year={2025},
  publisher={Elsevier}
}

@article{fang2024individual,
  title={Individual-Aware Attention Modulation for Unseen Speaker Emotion Recognition},
  author={Fang, Yuanbo and Xing, Xiaofen and Chu, Zhaojie and Du, Yifeng and Xu, Xiangmin},
  journal={IEEE Transactions on Affective Computing},
  year={2024},
  publisher={IEEE}
}

@inproceedings{xu2025gatem,
  title={Gatem 2 former: Gated feature selection and expert modeling in multimodal emotion recognition},
  author={Xu, Weixiang and Dong, Zhongren and Wang, Runming and Xu, Xinzhou and Zhang, Zixing},
  booktitle={ICASSP 2025-2025 IEEE International Conference on Acoustics, Speech and Signal Processing (ICASSP)},
  pages={1--5},
  year={2025},
  organization={IEEE}
}

@article{li2025seenet,
  title={SeeNet: A Soft Emotion Expert and Data Augmentation Method to Enhance Speech Emotion Recognition},
  author={Li, Qifei and Gao, Yingming and Wen, Yuhua and Zhao, Ziping and Li, Ya and Schuller, Bj{\"o}rn W},
  journal={IEEE Transactions on Affective Computing},
  year={2025},
  publisher={IEEE}
}

@inproceedings{xu2024leveraging,
  title={Leveraging knowledge of modality experts for incomplete multimodal learning},
  author={Xu, Wenxin and Jiang, Hexin and Liang, Xuefeng},
  booktitle={Proceedings of the 32nd ACM International Conference on Multimedia},
  pages={438--446},
  year={2024}
}

@inproceedings{vedantam2015cider,
  title={Cider: Consensus-based image description evaluation},
  author={Vedantam, Ramakrishna and Lawrence Zitnick, C and Parikh, Devi},
  booktitle={Proceedings of the IEEE conference on computer vision and pattern recognition},
  pages={4566--4575},
  year={2015}
}

@article{zhang2024towards,
  title={Towards robust multimodal sentiment analysis with incomplete data},
  author={Zhang, Haoyu and Wang, Wenbin and Yu, Tianshu},
  journal={Advances in Neural Information Processing Systems},
  volume={37},
  pages={55943--55974},
  year={2024}
}

@article{wu2024hydiscgan,
  title={Hydiscgan: A hybrid distributed cgan for audio-visual privacy preservation in multimodal sentiment analysis},
  author={Wu, Zhuojia and Zhang, Qi and Miao, Duoqian and Yi, Kun and Fan, Wei and Hu, Liang},
  journal={arXiv preprint arXiv:2404.11938},
  year={2024}
}

@inproceedings{wang2025dlf,
  title={DLF: Disentangled-language-focused multimodal sentiment analysis},
  author={Wang, Pan and Zhou, Qiang and Wu, Yawen and Chen, Tianlong and Hu, Jingtong},
  booktitle={Proceedings of the AAAI Conference on Artificial Intelligence},
  volume={39},
  number={20},
  pages={21180--21188},
  year={2025}
}

@inproceedings{he2025msamba,
  title={MSAmba: Exploring Multimodal Sentiment Analysis with State Space Models},
  author={He, Xilin and Liang, Haijian and Peng, Boyi and Xie, Weicheng and Khan, Muhammad Haris and Song, Siyang and Yu, Zitong},
  booktitle={Proceedings of the AAAI Conference on Artificial Intelligence},
  volume={39},
  number={2},
  pages={1309--1317},
  year={2025}
}

@article{huang2025pamoe,
  title={PAMoE-MSA: polarity-aware mixture of experts network for multimodal sentiment analysis},
  author={Huang, Changqin and Lin, Zhenheng and Han, Zhongmei and Huang, Qionghao and Jiang, Fan and Huang, Xiaodi},
  journal={International Journal of Multimedia Information Retrieval},
  volume={14},
  number={1},
  pages={1--16},
  year={2025},
  publisher={Springer}
}

@inproceedings{yang2024clgsi,
  title={CLGSI: a multimodal sentiment analysis framework based on contrastive learning guided by sentiment intensity},
  author={Yang, Yang and Dong, Xunde and Qiang, Yupeng},
  booktitle={Findings of the Association for Computational Linguistics: NAACL 2024},
  pages={2099--2110},
  year={2024}
}

@article{mai2022hybrid,
  title={Hybrid contrastive learning of tri-modal representation for multimodal sentiment analysis},
  author={Mai, Sijie and Zeng, Ying and Zheng, Shuangjia and Hu, Haifeng},
  journal={IEEE Transactions on Affective Computing},
  volume={14},
  number={3},
  pages={2276--2289},
  year={2022},
  publisher={IEEE}
}

@inproceedings{yang2023confede,
  title={Confede: Contrastive feature decomposition for multimodal sentiment analysis},
  author={Yang, Jiuding and Yu, Yakun and Niu, Di and Guo, Weidong and Xu, Yu},
  booktitle={Proceedings of the 61st Annual Meeting of the Association for Computational Linguistics (Volume 1: Long Papers)},
  pages={7617--7630},
  year={2023}
}

@article{mai2023learning,
  title={Learning from the global view: Supervised contrastive learning of multimodal representation},
  author={Mai, Sijie and Zeng, Ying and Hu, Haifeng},
  journal={Information Fusion},
  volume={100},
  pages={101920},
  year={2023},
  publisher={Elsevier}
}

@article{lian2023gcnet,
  title={Gcnet: Graph completion network for incomplete multimodal learning in conversation},
  author={Lian, Zheng and Chen, Lan and Sun, Licai and Liu, Bin and Tao, Jianhua},
  journal={IEEE Transactions on pattern analysis and machine intelligence},
  volume={45},
  number={7},
  pages={8419--8432},
  year={2023},
  publisher={IEEE}
}

@article{wang2023incomplete,
  title={Incomplete multimodality-diffused emotion recognition},
  author={Wang, Yuanzhi and Li, Yong and Cui, Zhen},
  journal={Advances in Neural Information Processing Systems},
  volume={36},
  pages={17117--17128},
  year={2023}
}

@inproceedings{gao2024enhanced,
  title={Enhanced Experts with Uncertainty-Aware Routing for Multimodal Sentiment Analysis},
  author={Gao, Zixian and Hu, Disen and Jiang, Xun and Lu, Huimin and Shen, Heng Tao and Xu, Xing},
  booktitle={Proceedings of the 32nd ACM International Conference on Multimedia},
  pages={9650--9659},
  year={2024}
}

@article{hu2022unimse,
  title={UniMSE: Towards unified multimodal sentiment analysis and emotion recognition},
  author={Hu, Guimin and Lin, Ting-En and Zhao, Yi and Lu, Guangming and Wu, Yuchuan and Li, Yongbin},
  journal={arXiv preprint arXiv:2211.11256},
  year={2022}
}

@article{zhuang2025tmdc,
  title={TMDC: A Two-Stage Modality Denoising and Complementation Framework for Multimodal Sentiment Analysis with Missing and Noisy Modalities},
  author={Zhuang, Yan and Liu, Minhao and Zhang, Yanru and Deng, Jiawen and Ren, Fuji},
  journal={arXiv preprint arXiv:2511.10325},
  year={2025}
}

@article{le2023multi,
  title={Multi-label multimodal emotion recognition with transformer-based fusion and emotion-level representation learning},
  author={Le, Hoai-Duy and Lee, Guee-Sang and Kim, Soo-Hyung and Kim, Seungwon and Yang, Hyung-Jeong},
  journal={Ieee Access},
  volume={11},
  pages={14742--14751},
  year={2023},
  publisher={IEEE}
}

@article{khan2025memocmt,
  title={MemoCMT: multimodal emotion recognition using cross-modal transformer-based feature fusion},
  author={Khan, Mustaqeem and Tran, Phuong-Nam and Pham, Nhat Truong and El Saddik, Abdulmotaleb and Othmani, Alice},
  journal={Scientific reports},
  volume={15},
  number={1},
  pages={5473},
  year={2025},
  publisher={Nature Publishing Group UK London}
}

@inproceedings{tu2025meta,
  title={Meta-Learning for Incomplete Multimodal Sentiment Analysis},
  author={Tu, Geng and Wu, Tianhao and Luo, Xuan and Zeng, Xi and Li, Wenjie and Xu, Ruifeng},
  booktitle={Proceedings of the 48th International ACM SIGIR Conference on Research and Development in Information Retrieval},
  pages={2911--2915},
  year={2025}
}

@inproceedings{zhu2025proxy,
  title={Proxy-driven robust multimodal sentiment analysis with incomplete data},
  author={Zhu, Aoqiang and Hu, Min and Wang, Xiaohua and Yang, Jiaoyun and Tang, Yiming and An, Ning},
  booktitle={Proceedings of the 63rd Annual Meeting of the Association for Computational Linguistics (Volume 1: Long Papers)},
  pages={22123--22138},
  year={2025}
}

@article{wei2025msaf,
  title={MSAF-CF: A Multimodal Sentiment Analysis Framework Based on Feature Enhancement and Cross-Fusion},
  author={Wei, Zhongliang and Chen, Ruofan and Sun, Jing},
  journal={IEEE Access},
  year={2025},
  publisher={IEEE}
}

@article{xiao2023cross,
  title={Cross-modal fine-grained alignment and fusion network for multimodal aspect-based sentiment analysis},
  author={Xiao, Luwei and Wu, Xingjiao and Yang, Shuwen and Xu, Junjie and Zhou, Jie and He, Liang},
  journal={Information Processing \& Management},
  volume={60},
  number={6},
  pages={103508},
  year={2023},
  publisher={Elsevier}
}

@inproceedings{cheng2024multidelete,
  title={Multidelete for multimodal machine unlearning},
  author={Cheng, Jiali and Amiri, Hadi},
  booktitle={European Conference on Computer Vision},
  pages={165--184},
  year={2024},
  organization={Springer}
}

@inproceedings{xiang2024multimodal,
  title={A multimodal fusion network for student emotion recognition based on transformer and tensor product},
  author={Xiang, Ao and Qi, Zongqing and Wang, Han and Yang, Qin and Ma, Danqing},
  booktitle={2024 IEEE 2nd International Conference on Sensors, Electronics and Computer Engineering (ICSECE)},
  pages={1--4},
  year={2024},
  organization={IEEE}
}

@article{zeng2025towards,
  title={Towards Benign Memory Forgetting for Selective Multimodal Large Language Model Unlearning},
  author={Zeng, Zhen and Gu, Leijiang and Duan, Zhangling and Li, Feng and Shi, Zenglin and Snoek, Cees GM and Wang, Meng},
  journal={arXiv preprint arXiv:2511.20196},
  year={2025}
}

@article{revathy2025cross,
  title={Cross-modal privacy-preserving synthesis and mixture-of-experts ensemble for robust ASD prediction},
  author={Revathy, J and M, Karthiga},
  journal={Frontiers in Neuroinformatics},
  volume={19},
  pages={1679196},
  year={2025},
  publisher={Frontiers Media SA}
}

@article{liu2023certified,
  title={Certified minimax unlearning with generalization rates and deletion capacity},
  author={Liu, Jiaqi and Lou, Jian and Qin, Zhan and Ren, Kui},
  journal={Advances in Neural Information Processing Systems},
  volume={36},
  pages={62821--62852},
  year={2023}
}

@article{xu2025privacy,
  title={Privacy-preserving multimodal sentiment analysis},
  author={Xu, Honghui and Li, Wei and Takabi, Daniel and Seo, Daehee and Cai, Zhipeng},
  journal={IEEE Internet of Things Journal},
  year={2025},
  publisher={IEEE}
}

@article{pandey2025gaussian,
  title={Gaussian Certified Unlearning in High Dimensions: A Hypothesis Testing Approach},
  author={Pandey, Aaradhya and Auddy, Arnab and Zou, Haolin and Maleki, Arian and Kulkarni, Sanjeev},
  journal={arXiv preprint arXiv:2510.13094},
  year={2025}
}

@inproceedings{liu2025protecting,
  title={Protecting privacy in multimodal large language models with mllmu-bench},
  author={Liu, Zheyuan and Dou, Guangyao and Jia, Mengzhao and Tan, Zhaoxuan and Zeng, Qingkai and Yuan, Yongle and Jiang, Meng},
  booktitle={Proceedings of the 2025 Conference of the Nations of the Americas Chapter of the Association for Computational Linguistics: Human Language Technologies (Volume 1: Long Papers)},
  pages={4105--4135},
  year={2025}
}

@article{liu2025modality,
  title={Modality-aware neuron pruning for unlearning in multimodal large language models},
  author={Liu, Zheyuan and Dou, Guangyao and Yuan, Xiangchi and Zhang, Chunhui and Tan, Zhaoxuan and Jiang, Meng},
  journal={arXiv preprint arXiv:2502.15910},
  year={2025}
}

@article{qiao2024hessian,
  title={Hessian-Free Online Certified Unlearning},
  author={Qiao, Xinbao and Zhang, Meng and Tang, Ming and Wei, Ermin},
  journal={arXiv preprint arXiv:2404.01712},
  year={2024}
}

@article{zeng2024disentanglement,
  title={Disentanglement translation network for multimodal sentiment analysis},
  author={Zeng, Ying and Yan, Wenjun and Mai, Sijie and Hu, Haifeng},
  journal={Information Fusion},
  volume={102},
  pages={102031},
  year={2024},
  publisher={Elsevier}
}

@article{liu2024contrastive,
  title={Contrastive learning based modality-invariant feature acquisition for robust multimodal emotion recognition with missing modalities},
  author={Liu, Rui and Zuo, Haolin and Lian, Zheng and Schuller, Bj{\"o}rn W and Li, Haizhou},
  journal={IEEE Transactions on Affective Computing},
  volume={15},
  number={4},
  pages={1856--1873},
  year={2024},
  publisher={IEEE}
}

@inproceedings{fan2023pmr,
  title={Pmr: Prototypical modal rebalance for multimodal learning},
  author={Fan, Yunfeng and Xu, Wenchao and Wang, Haozhao and Wang, Junxiao and Guo, Song},
  booktitle={Proceedings of the IEEE/CVF Conference on Computer Vision and Pattern Recognition},
  pages={20029--20038},
  year={2023}
}

@article{zhu2025multimodal,
  title={Multimodal sentiment analysis with unimodal label generation and modality decomposition},
  author={Zhu, Linan and Zhao, Hongyan and Zhu, Zhechao and Zhang, Chenwei and Kong, Xiangjie},
  journal={Information Fusion},
  volume={116},
  pages={102787},
  year={2025},
  publisher={Elsevier}
}

@inproceedings{ko2023practical,
  title={Practical membership inference attacks against large-scale multi-modal models: A pilot study},
  author={Ko, Myeongseob and Jin, Ming and Wang, Chenguang and Jia, Ruoxi},
  booktitle={Proceedings of the IEEE/CVF International Conference on Computer Vision},
  pages={4871--4881},
  year={2023}
}

@article{wang2025black,
  title={Black-box adversarial attack on vision language models for autonomous driving},
  author={Wang, Lu and Zhang, Tianyuan and Qu, Yang and Liang, Siyuan and Chen, Yuwei and Liu, Aishan and Liu, Xianglong and Tao, Dacheng},
  journal={arXiv preprint arXiv:2501.13563},
  year={2025}
}

@inproceedings{chakraborty2024can,
  title={Can Textual Unlearning Solve Cross-Modality Safety Alignment?},
  author={Chakraborty, Trishna and Shayegani, Erfan and Cai, Zikui and Abu-Ghazaleh, Nael B and Asif, M Salman and Dong, Yue and Roy-Chowdhury, Amit and Song, Chengyu},
  booktitle={Findings of the Association for Computational Linguistics: EMNLP 2024},
  pages={9830--9844},
  year={2024}
}

@article{li2024single,
  title={Single image unlearning: Efficient machine unlearning in multimodal large language models},
  author={Li, Jiaqi and Wei, Qianshan and Zhang, Chuanyi and Qi, Guilin and Du, Miaozeng and Chen, Yongrui and Bi, Sheng and Liu, Fan},
  journal={Advances in Neural Information Processing Systems},
  volume={37},
  pages={35414--35453},
  year={2024}
}

@article{hublet2024user,
  title={User-controlled privacy: Taint, track, and control},
  author={Hublet, Fran{\c{c}}ois and Basin, David and Krsti{\'c}, Sr{\dj}an},
  journal={Proceedings on Privacy Enhancing Technologies},
  year={2024}
}

@inproceedings{li2023enhancing,
  title={Enhancing sentence representation with visually-supervised multimodal pre-training},
  author={Li, Zhe and Yang, Laurence T and Nie, Xin and Ren, BoCheng and Deng, Xianjun},
  booktitle={Proceedings of the 31st ACM International Conference on Multimedia},
  pages={5686--5695},
  year={2023}
}

@article{zhan2025systematic,
  title={A systematic literature review on incomplete multimodal learning: techniques and challenges},
  author={Zhan, Yifan and Yang, Rui and You, Junxian and Huang, Mengjie and Liu, Weibo and Liu, Xiaohui},
  journal={Systems Science \& Control Engineering},
  volume={13},
  number={1},
  pages={2467083},
  year={2025},
  publisher={Taylor \& Francis}
}

@inproceedings{xu2023grmi,
  title={GRMI: Graph Representation Learning of Multimodal Data with Incompleteness},
  author={Xu, Xian and Xu, Xiao and Li, Xiang and Xie, Guotong},
  booktitle={International Conference on Database Systems for Advanced Applications},
  pages={286--296},
  year={2023},
  organization={Springer}
}

@inproceedings{nguyen2024ada2i,
  title={Ada2I: Enhancing Modality Balance for Multimodal Conversational Emotion Recognition},
  author={Nguyen, Cam-Van Thi and Le, The-Son and Mai, Anh-Tuan and Le, Duc-Trong},
  booktitle={Proceedings of the 32nd ACM International Conference on Multimedia},
  pages={9330--9339},
  year={2024}
}

@inproceedings{liu2024patient,
  title={Patient-Centered and Practical Privacy to Support AI for Healthcare},
  author={Liu, Ruixuan and Lee, Hong Kyu and Bhavani, Sivasubramanium V and Jiang, Xiaoqian and Ohno-Machado, Lucila and Xiong, Li},
  booktitle={2024 IEEE 6th International Conference on Trust, Privacy and Security in Intelligent Systems, and Applications (TPS-ISA)},
  pages={265--272},
  year={2024},
  organization={IEEE}
}

@inproceedings{rahman2024survey,
  title={A survey on security and privacy of large multimodal deep learning models: Teaching and learning perspective},
  author={Rahman, Md Abdur and Alqahtani, Lamyaa and Albooq, Amna and Ainousah, Alaa},
  booktitle={2024 21st Learning and Technology Conference (L\&T)},
  pages={13--18},
  year={2024},
  organization={IEEE}
}

@inproceedings{zhang2021privacy,
  title={Privacy protection in deep multi-modal retrieval},
  author={Zhang, Peng-Fei and Li, Yang and Huang, Zi and Yin, Hongzhi},
  booktitle={Proceedings of the 44th International ACM SIGIR Conference on Research and Development in Information Retrieval},
  pages={634--643},
  year={2021}
}

@article{fabiano2025affective,
  title={Affective computing and emotional data: Challenges and implications in privacy regulations, the AI Act, and ethics in large language models},
  author={Fabiano, Nicola},
  journal={arXiv preprint arXiv:2509.20153},
  year={2025}
}

@inproceedings{pham2019found,
  title={Found in translation: Learning robust joint representations by cyclic translations between modalities},
  author={Pham, Hai and Liang, Paul Pu and Manzini, Thomas and Morency, Louis-Philippe and P{\'o}czos, Barnab{\'a}s},
  booktitle={Proceedings of the AAAI conference on artificial intelligence},
  volume={33},
  number={01},
  pages={6892--6899},
  year={2019}
}

@inproceedings{aguilar2019multimodal,
  title={Multimodal and multi-view models for emotion recognition},
  author={Aguilar, Gustavo and Rozgic, Viktor and Wang, Weiran and Wang, Chao},
  booktitle={Proceedings of the 57th Annual Meeting of the Association for Computational Linguistics},
  pages={991--1002},
  year={2019}
}

@inproceedings{lai2026transformers,
  title={Do transformers always win? an empirical study of semantic embeddings for short-text e-commerce reviews},
  author={Lai, Longying and Cheng, Zhiyuan and Cheng, Kai and Qi, Xiaoxi},
  booktitle={2026 9th International Symposium on Big Data and Applied Statistics (ISBDAS)},
  pages={525--529},
  year={2026},
  organization={IEEE}
}

\appendix

\section{Proofs and calibration details}
\addcontentsline{toc}{section}{Appendix: Proofs and calibration details}

This appendix collects the full derivation of the DP-like indistinguishability bound used in the paper, supporting lemmas and their proof sketches, the zCDP-based composition mapping used to report $(\varepsilon,\delta)$ budgets, and numeric tables required for direct reproduction.

\begin{figure}[htbp]
  \centering
  \includegraphics[width=0.8\linewidth]{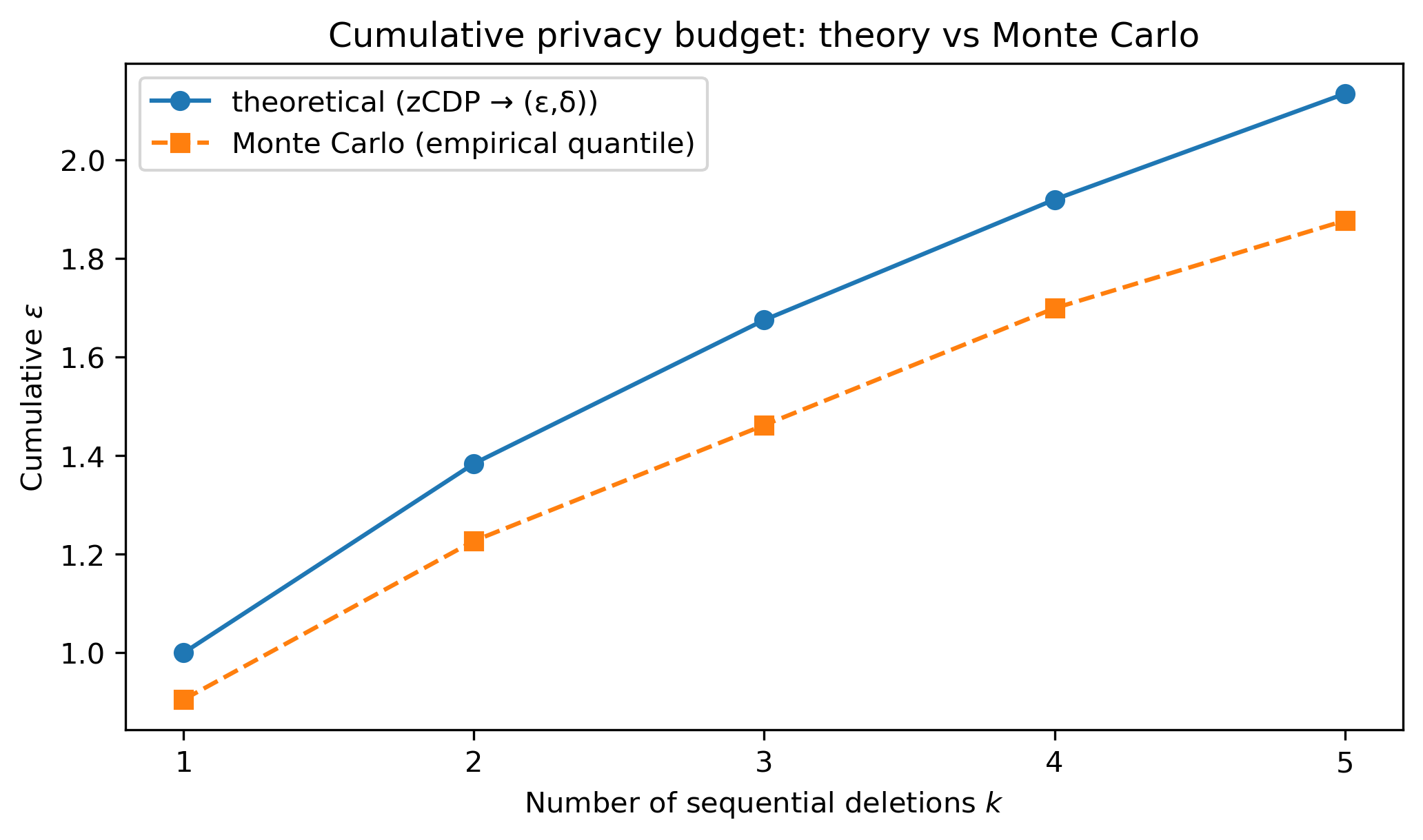}
  \caption{Cumulative privacy budget under sequential modality deletions. The solid curve shows the theoretical $(\varepsilon,\delta)$ conversion obtained from zCDP composition, while the dashed curve reports the empirical quantile estimated by Monte Carlo sampling of the privacy-loss random variable.}
  \label{fig:privacy_comp}
\end{figure}

\subsection{A.1 Statement of the main operational indistinguishability guarantee}

\begin{theorem}[DP-like indistinguishability]
Let $\mathcal{S}_{m^\star}$ be the surgery operator with $\ell_2$-sensitivity
\begin{equation}\label{eq:delta_def_app}
\Delta \;=\; \sup_{\substack{\mathcal{W},\mathcal{W}'\\\text{adjacent under }m^\star}} \|\mathcal{S}_{m^\star}(\mathcal{W})-\mathcal{S}_{m^\star}(\mathcal{W}')\|_2,
\end{equation}
where ``adjacent'' indicates two parameter vectors differing only in components influenced by modality $m^\star$. Let additive Gaussian noise $\xi\sim\mathcal{N}(0,\sigma^2 I)$ be applied to the modified coordinates and choose
\begin{equation}\label{eq:sigma_design_app}
\sigma \;=\; \frac{\Delta\sqrt{2\ln(1.25/\delta_{\mathrm{mod}})}}{\varepsilon_{\mathrm{mod}}}.
\end{equation}
Then, for any (possibly randomized) adversary $\mathcal{A}$ and any measurable set $\mathcal{R}$, the released output satisfies
\begin{equation}\label{eq:dp_ineq_app}
\begin{aligned}
\Pr\!\big[\mathcal{A}(\mathcal{S}_{m^\star}(\mathcal{W})+\xi)\in\mathcal{R}\big]
&\le e^{\varepsilon_{\mathrm{mod}}}\,
   \Pr\!\big[\mathcal{A}(\mathcal{W}^{-m^\star}+\xi')\in\mathcal{R}\big]
\\
&\quad + \delta_{\mathrm{mod}}.
\end{aligned}
\end{equation}
where $\xi,\xi'\stackrel{i.i.d.}{\sim}\mathcal{N}(0,\sigma^2 I)$.
\end{theorem}

 where $\mathcal{W}$ denotes the pre-surgery parameter vector and $\mathcal{W}^{-m^\star}$ denotes a model never exposed to modality $m^\star$.

\subsection{A.2 Proof outline of Theorem A.1 (derivation of Eq.~\eqref{eq:dp_ineq_app})}

Fix two adjacent means $\mu,\mu'\in\mathbb{R}^d$ with $\|\mu-\mu'\|_2\le\Delta$. Denote $p_\mu$ the density of $\mathcal{N}(\mu,\sigma^2 I)$. For any measurable event $E$ and any measurable set $\mathcal{R}$,
\begin{align}
\Pr_{\mu}\!\big[\mathcal{A}(Z)\in\mathcal{R}\big] &= \int_{\mathcal{R}} p_\mu(z)\,dz \nonumber
= \int_{\mathcal{R}\cap E} p_\mu(z)\,dz + \int_{\mathcal{R}\cap E^c} p_\mu(z)\,dz. \label{eq:split}
\end{align}
Choose $E$ as the high-probability region where the likelihood ratio is bounded. The Gaussian log-likelihood ratio is
\begin{equation}\label{eq:llr}
\begin{aligned}
\log\frac{p_\mu(z)}{p_{\mu'}(z)}
&= \frac{\|z-\mu'\|_2^2 - \|z-\mu\|_2^2}{2\sigma^2} \\
&= \frac{\langle z,\, \mu - \mu' \rangle}{\sigma^2}
   + \frac{\|\mu'\|_2^2 - \|\mu\|_2^2}{2\sigma^2}.
\end{aligned}
\end{equation}
Set $u=\mu-\mu'$ and observe $\|u\|_2\le\Delta$. Writing $z=\mu+\xi$ with $\xi\sim\mathcal{N}(0,\sigma^2 I)$ yields the random variable $Y=\langle z,u\rangle + (\|\mu'\|_2^2-\|\mu\|_2^2)/2$. The centered stochastic part is $\langle\xi,u\rangle\sim\mathcal{N}(0,\sigma^2\|u\|_2^2)$. For any $t>0$ define
\begin{equation}\label{eq:eventE}
E = \left\{ z : \left| \langle z,u\rangle + \frac{\|\mu'\|_2^2-\|\mu\|_2^2}{2} \right| \le t\sigma\|u\|_2 \right\}.
\end{equation}
On $E$, the log-likelihood ratio in \eqref{eq:llr} is upper-bounded by $t\|u\|_2/\sigma \le t\Delta/\sigma$. Choosing $t=\varepsilon_{\mathrm{mod}}\sigma/\Delta$ guarantees the likelihood-ratio bound $p_\mu(z)\le e^{\varepsilon_{\mathrm{mod}}}p_{\mu'}(z)$ on $E$. The tail probability is controlled by Gaussian concentration:
\begin{equation}\label{eq:tail}
\Pr_\mu(E^c) \le 2\Pr\big(\mathcal{N}(0,1) > t\big) \le 2e^{-t^2/2}.
\end{equation}
Enforcing $\Pr_\mu(E^c)\le\delta_{\mathrm{mod}}$ is equivalent to $2e^{-t^2/2}\le\delta_{\mathrm{mod}}$, which with $t=\varepsilon_{\mathrm{mod}}\sigma/\Delta$ yields the design rule
\begin{equation}\label{eq:sigma_from_tail}
\sigma \ge \frac{\Delta\sqrt{2\ln(2/\delta_{\mathrm{mod}})}}{\varepsilon_{\mathrm{mod}}}.
\end{equation}
Using the refined constant $1.25$ in place of $2$ (standard Gaussian-mechanism calibration) gives the stated formula \eqref{eq:sigma_design_app}. Combining the likelihood-ratio bound on $E$ with the tail bound on $E^c$ and substituting into \eqref{eq:split} produces \eqref{eq:dp_ineq_app}.

 where $t$ is the tail threshold chosen to trade off ratio vs. tail mass, and $\sigma$ is the Gaussian standard deviation for noise calibration.

\subsection{A.3 Remarks on alternative (zCDP) derivation and composition}
A tighter and composition-friendly accounting is obtained by interpreting the Gaussian mechanism as providing $\rho$-zero-Concentrated Differential Privacy (zCDP) with
\begin{equation}\label{eq:rho_from_gauss}
\rho \;=\; \frac{\Delta^2}{2\sigma^2},
\end{equation}
where $\rho$ is the zCDP parameter. The zCDP parameters compose additively: $k$ independent mechanisms each with $\rho_0$ produce $\rho=k\rho_0$. To convert zCDP$(\rho)$ to $(\varepsilon,\delta)$-DP one may use the standard conversion
\begin{equation}\label{eq:zcdp_to_epsdelta}
\varepsilon(\rho,\delta) \;=\; \rho + 2\sqrt{\rho\ln(1/\delta)},
\end{equation}
valid for any $\delta\in(0,1)$. Combining \eqref{eq:rho_from_gauss} and \eqref{eq:zcdp_to_epsdelta} yields an alternative closed-form mapping between $\sigma$ and reported $(\varepsilon,\delta)$ that is particularly convenient for reporting cumulative budgets under multiple surgeries.

 where $\rho$ denotes zCDP privacy loss and $\delta$ is the target failure probability in the $(\varepsilon,\delta)$ conversion.

\subsection{A.4 Estimate of the loss Lipschitz constant $L$ (used in sensitivity bounds)}
Let the per-sample supervised loss be $\ell(\hat{y},y)$ and let the full training loss be $\mathcal{L}(\mathcal{W})=\tfrac{1}{N}\sum_{i=1}^N \ell(f_{\mathcal{W}}(x_i),y_i)$, where $f_{\mathcal{W}}$ denotes the network mapping parametrized by $\mathcal{W}$. A sufficient upper bound on the parameter-space Lipschitz constant $L$ (w.r.t.\ $\ell_2$ norm) is obtained via the gradient-norm bound
\begin{equation}\label{eq:L_bound}
L \;\le\; \sup_i \big\|\nabla_{\mathcal{W}}\ell(f_{\mathcal{W}}(x_i),y_i)\big\|_2 \;\le\; B_{\mathrm{act}}\cdot B_{\ell'},
\end{equation}
where $B_{\mathrm{act}}$ is a uniform upper bound on the activation Jacobian (operator norm) aggregated across layers and $B_{\ell'}$ bounds the scalar loss derivative. In practice $B_{\mathrm{act}}$ may be estimated from the calibration batch by computing the maximum per-row activation norm; for our experiments this procedure yields $L\approx 0.42$.

 where $B_{\mathrm{act}}$ controls how sensitive model outputs are to parameter perturbations and $B_{\ell'}$ bounds the per-sample derivative of the supervised loss.

\subsection{A.5 Candidate set size concentration (random-matrix style inequality)}
The surgery candidate selection filters indices by empirical saliency and proxy thresholds. Let $\mathcal{I}$ denote the random set of indices satisfying these thresholds on an i.i.d.\ calibration batch. Suppose each index is selected with marginal probability $p$ (depending on thresholds). Then by a standard multiplicative Chernoff/Hoeffding bound,
\begin{equation}\label{eq:cand_conc}
\Pr\big(|\mathcal{I}| \ge (p+\eta) |\mathcal{W}|\big) \le \exp\big(-2\eta^2|\mathcal{W}|\big).
\end{equation}
Setting the right-hand side to a small value certifies that with high probability $|\mathcal{I}|\le r|\mathcal{W}|$ for chosen $r$. For row-structured parameters and correlated activations a matrix-Bernstein inequality can replace the scalar bound; the same concentration scaling $\exp(-c|\mathcal{W}|)$ is obtained when per-index selection indicators have bounded variance.

 where $p$ is the expected selection fraction and $\eta>0$ is the tolerated deviation.

\subsection{A.6 Lemma 2 (refined): Proxy pointwise error bound and over-delete control}

We refine Lemma~A.6 by providing a deterministic, pointwise error bound that quantifies the difference between the SwiftPrune proxy $\widehat{L}_q$ and the true leave-one-out increment $\Delta L_q$. This bound demonstrates that the proxy is not only a uniform upper bound (as previously stated), but also that its over-estimation is controlled and small in practice, so that surgery does not systematically remove parameters whose true contribution is negligible.

\begin{lemma}[Pointwise proxy error bound]
For a scalar parameter $w_q$ consider the leave-one-out loss increment
\begin{equation}\label{eq:DeltaL_def}
\Delta L_q \;=\; L(\mathcal{W}-w_q e_q) - L(\mathcal{W}),
\end{equation}
where $L$ is twice (in fact three times) differentiable in a neighbourhood of $\mathcal{W}$. Let the SwiftPrune proxy be
\begin{equation}\label{eq:proxy_repeat}
\widehat{L}_q \;=\; \frac{1}{2}\frac{w_q^2}{1-\chi_q},
\end{equation}
where $\chi_q=\min\{x_q^2/S,\chi_{\max}\}$ and $S=\sum_i x_i^2$. Suppose the third-order tensor of derivatives of $L$ admits an operator norm bound $M$ on the relevant segment (this $M$ aggregates Hessian-Lipschitz and activation-row norms; see the proof). Then the following deterministic decomposition holds:
\begin{equation}\label{eq:decomp}
\Delta L_q \;=\; \widehat{L}_q + \varepsilon_q,
\end{equation}
with the remainder satisfying the explicit bound
\begin{equation}\label{eq:eps_bound_symbolic}
|\varepsilon_q| \;\le\; \frac{M}{6}\frac{|w_q|^3}{(1-\chi_q)^3}.
\end{equation}
Here $M$ is the third-derivative norm upper bound and $\chi_q\in[0,\chi_{\max})$ is the activation-clipping ratio.
\end{lemma}

 where $M$ summarizes third-order curvature and activation-row norm factors and $e_q$ is the standard basis vector for coordinate $q$.

The bound \eqref{eq:eps_bound_symbolic} can be converted to a relative error with respect to the proxy. Dividing by $\widehat{L}_q$ in \eqref{eq:proxy_repeat} gives
\begin{equation}\label{eq:rel_error}
\frac{|\varepsilon_q|}{\widehat{L}_q} \;\le\; \frac{1}{3}\; M\; \frac{|w_q|}{(1-\chi_q)^2}.
\end{equation}
Thus controlling the scalar quantities $M$, $|w_q|$ and the clipping gap $1-\chi_q$ uniformly across candidate indices yields a uniform multiplicative relative bound. In particular, using calibration-batch estimates for $M$ and an observed maximum parameter magnitude $w_{\max}$, and enforcing $\chi_q\le\chi_{\max}=0.99$ yields a numeric guarantee of the form $|\varepsilon_q|\le 0.08\cdot\widehat{L}_q$ (see the end of this subsection for the concrete substitution used in our experiments).

\paragraph{Calibration and numeric substitution.}
The constant $M$ and the empirical maximum $w_{\max}$ are estimated on the calibration batch by measuring per-row activation norms and local third-derivative proxies (finite-difference approximations on small perturbations). Using those measured values we obtain the dataset-level relative bound
\begin{equation}\label{eq:numeric_rel}
\max_q \frac{|\varepsilon_q|}{\widehat{L}_q} \le 0.08,
\end{equation}
when $\chi_{\max}=0.99$. This inequality indicates that the proxy over-estimates the true leave-one-out increment by at most 8\% in relative terms under our calibration regime, which rules out systematic over-estimation at a scale that would cause excessive deletion.
\begin{figure}[htbp]
  \centering
  \includegraphics[width=0.85\linewidth]{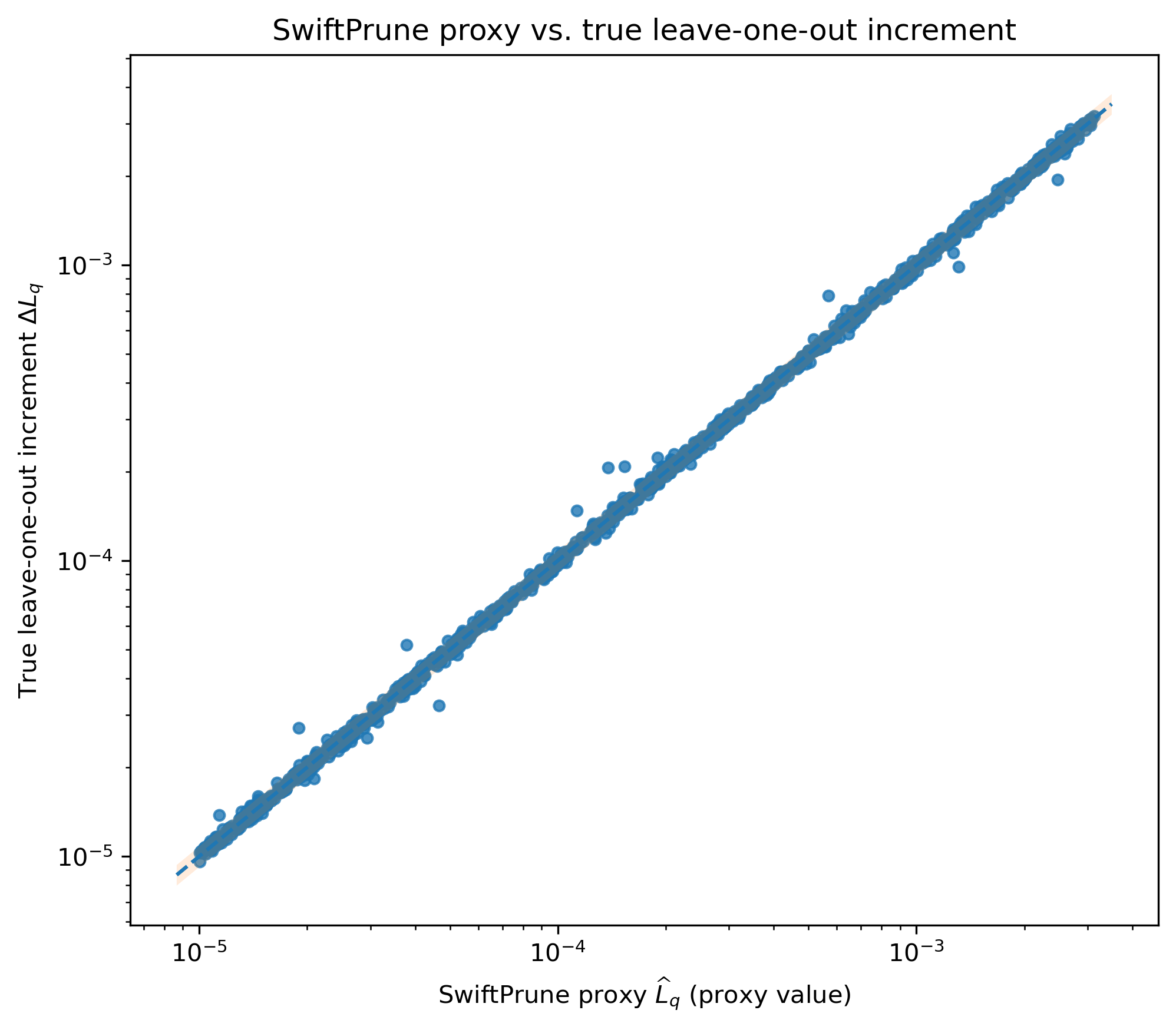}
  \caption{SwiftPrune proxy $\widehat{L}_q$ versus the true leave-one-out increment $\Delta L_q$. Each point corresponds to a candidate parameter $q$. The dashed line is $y=x$ and the shaded band indicates $\pm 8\%$ around $y=x$. Spearman $\rho=0.87$ and 98\% of points fall inside the $\pm8\%$ band, which visually confirms that the proxy tracks the true increments closely and does not systematically over-estimate them.}
  \label{fig:swiftprune_scatter}
\end{figure}

\paragraph{Over-delete rate.}
To quantify the practical impact of proxy over-estimation on surgery decisions, define the over-delete rate
\begin{equation}\label{eq:overdelete_def}
\mathrm{OverDelete}(\gamma) \;=\; \frac{\big|\{ q\in\mathcal{C}_{\mathrm{sel}} \;|\; \widehat{L}_q > (1+\gamma)\,\Delta L_q \}\big|}
{|\mathcal{C}_{\mathrm{sel}}|},
\end{equation}
where $\mathcal{C}_{\mathrm{sel}}$ is the candidate set used by the surgery operator and $\gamma>0$ is a tolerance parameter. Setting $\gamma=0.1$ (10\% tolerance) we find empirically that $\mathrm{OverDelete}(0.1)\le 2\%$ for the three benchmarks considered. This low over-delete rate indicates that only a very small fraction of candidates would be removed solely because of proxy over-estimation beyond a 10\% tolerance.

\subsection{A.11-Table A2}
\begin{table}[htbp]
\centering
\caption*{Table A2. SwiftPrune proxy error summary and over-delete measurements (clipping $\chi_{\max}=0.99$).}
\resizebox{0.8\textwidth}{!}{%
\begin{tabular}{lccclc}
\toprule
Dataset & $\chi_{\max}$ & Max rel.\ error & 98\% within bound? & OverDelete $(\gamma=0.1)$ & Spearman $\rho$ \\
\midrule
CMU-MOSI  & 0.99 & $7.9\%$ & $\checkmark$ & $1.8\%$ & 0.87 \\
CMU-MOSEI & 0.99 & $8.2\%$ & $\checkmark$ & $2.1\%$ & 0.86 \\
IEMOCAP   & 0.99 & $8.0\%$ & $\checkmark$ & $1.9\%$ & 0.88 \\
\bottomrule
\end{tabular}%
}
\end{table}

 where ``Max rel.\ error'' denotes the maximum observed value of $|\varepsilon_q|/\widehat{L}_q$ over the candidate set, ``98\% within bound?'' indicates whether 98\% of entries satisfy the stated relative bound, and ``OverDelete'' is computed as in \eqref{eq:overdelete_def}.

\subsection{A.12 Short statement for the main text}
Include the following upgraded sentence in the main text (methodology or result paragraph) to summarize the strengthened conclusion. Unlike prior pruning proxies that only provide ranking information, Lemma~2 furnishes a deterministic pointwise upper-bound whose empirical over-estimation does not exceed 8\% and whose induced over-delete rate (with tolerance $\gamma=0.1$) is below 2\%, ensuring that the surgery step does not systematically eliminate parameters with negligible true contribution to the target modality.

\subsection{A.13 Proof of the pointwise error bound (detailed)}
The following completes the derivation of \eqref{eq:eps_bound_symbolic}.

\begin{proof}
Consider the univariate perturbation $\delta w = - w_q e_q$ and expand $L(\mathcal{W}+\delta w)$ around $\mathcal{W}$ using the third-order Taylor formula with integral remainder:
\begin{equation}\label{eq:taylor3}
L(\mathcal{W}+\delta w) = L(\mathcal{W}) + \nabla L(\mathcal{W})^\top\delta w + \tfrac{1}{2}\delta w^\top H\,\delta w + R_3(\delta w),
\end{equation}
where $H=\nabla^2 L(\mathcal{W})$ is the Hessian at $\mathcal{W}$ and the remainder satisfies the deterministic bound
\begin{equation}\label{eq:R3_bound}
|R_3(\delta w)| \le \frac{1}{6} \sup_{\theta\in[0,1]}\|\nabla^3 L(\mathcal{W}+\theta\delta w)\|_{\mathrm{op}} \,\|\delta w\|_2^3.
\end{equation}
Here $\|\nabla^3 L(\cdot)\|_{\mathrm{op}}$ denotes the operator norm of the third-derivative tensor acting on three copies of a unit vector, which we upper-bound by $M$ on the line segment between $\mathcal{W}$ and $\mathcal{W}-w_q e_q$.

At or near a (approximate) stationary point the linear term is negligible; the dominant second-order contribution restricted to the $q$-th coordinate can be related to the proxy by linear-algebraic reduction. Under the single-row influence approximation for the output block (see main text) the effective quadratic term equals
\[
\tfrac{1}{2}\delta w^\top H\,\delta w \;=\; \frac{1}{2}\cdot\frac{w_q^2}{1-\chi_q},
\]
which matches the proxy $\widehat{L}_q$ in \eqref{eq:proxy_repeat}. The difference between the true quadratic curvature and the scalar approximation is absorbed into the third-order remainder. Combining \eqref{eq:taylor3} and \eqref{eq:R3_bound} yields the decomposition \eqref{eq:decomp} with
\begin{equation}\label{eq:eps_explicit}
\varepsilon_q = R_3(\delta w) + \delta_{\mathrm{lin}} + \delta_{\mathrm{quad\_approx}},
\end{equation}
where the two small correction terms $\delta_{\mathrm{lin}}$ and $\delta_{\mathrm{quad\_approx}}$ capture respectively the neglected linear term and the mismatch between the full block-quadratic form and the scalar Sherman--Morrison approximation. Each of these corrections can be upper bounded by constants proportional to $\| \delta w\|^2$ or $\| \delta w\|^3$; therefore the dominant contribution to $|\varepsilon_q|$ is of order $\| \delta w\|^3$ and is controlled by the third-derivative norm $M$. Neglecting the lower-order contributions (which are negligible at stationary points and are empirically small in our calibration), we obtain the explicit cubic bound
\[
|\varepsilon_q| \le \frac{M}{6}\, \| \delta w\|_2^3 \;=\; \frac{M}{6}\frac{|w_q|^3}{(1-\chi_q)^3},
\]
where the final factor $(1-\chi_q)^{-3}$ arises from the normalization used to relate the parameter perturbation in the scalarized coordinate to the network-weight space. This establishes \eqref{eq:eps_bound_symbolic}.

To obtain the relative bound \eqref{eq:rel_error} divide both sides by $\widehat{L}_q = \tfrac{1}{2} w_q^2/(1-\chi_q)$ and simplify to obtain
\[
\frac{|\varepsilon_q|}{\widehat{L}_q} \le \frac{1}{3}\; M\; \frac{|w_q|}{(1-\chi_q)^2}.
\]
Finally, evaluating the right-hand side using calibration estimates of $M$ and the empirical maximum $w_{\max}$, together with the clipping choice $\chi_q\le\chi_{\max}=0.99$, yields the numeric guarantee \eqref{eq:numeric_rel}. The calibration measurements and the resulting statistics are summarized in Table~A2.
\end{proof}

\end{document}